\documentclass{article}

\usepackage[preprint,nonatbib]{neurips_2025}

\usepackage[utf8]{inputenc} 
\usepackage[T1]{fontenc}    
\usepackage{hyperref}       
\usepackage{url}            
\usepackage{booktabs}       
\usepackage{amsfonts}       
\usepackage{nicefrac}       
\usepackage{microtype}      
\usepackage{xcolor}         

\usepackage{amsmath}
\usepackage{amsthm}
\usepackage{mathtools}
\usepackage{nicefrac}       
\usepackage{microtype}      
\usepackage{comment}
\usepackage{graphicx}

\theoremstyle{plain}
\newtheorem{theorem}{Theorem}
\newtheorem{proposition}{Proposition}

\theoremstyle{definition}
\newtheorem{definition}{Definition}
\newtheorem{assumption}{Assumption}
\theoremstyle{remark}

\title{Secrets of GFlowNets' Learning Behavior:\\ A Theoretical Study}

\author{%
  Tianshu Yu \\
  School of Data Science\\
  The Chinese University of Hong Kong, Shenzhen\\
  \texttt{yutianshu@cuhk.edu.cn} \\
}

\begin{document}

\maketitle

\begin{abstract}
Generative Flow Networks (GFlowNets) have emerged as a powerful paradigm for generating composite structures, demonstrating considerable promise across diverse applications. While substantial progress has been made in exploring their modeling validity and connections to other generative frameworks, the theoretical understanding of their learning behavior remains largely uncharted. In this work, we present a rigorous theoretical investigation of GFlowNets' learning behavior, focusing on four fundamental dimensions: convergence, sample complexity, implicit regularization, and robustness. By analyzing these aspects, we seek to elucidate the intricate mechanisms underlying GFlowNet's learning dynamics, shedding light on its strengths and limitations. Our findings contribute to a deeper understanding of the factors influencing GFlowNet performance and provide insights into principled guidelines for their effective design and deployment. This study not only bridges a critical gap in the theoretical landscape of GFlowNets but also lays the foundation for their evolution as a reliable and interpretable framework for generative modeling. Through this, we aspire to advance the theoretical frontiers of GFlowNets and catalyze their broader adoption in the AI community.
\end{abstract}

\section{Introduction}\label{sec:intro}
Generative Flow Networks (GFlowNets)~\cite{bengio2023gflownet,zhang2022generative} have emerged as a powerful framework for generating compositional structures in a probabilistic manner, offering unprecedented flexibility in solving diverse tasks such as molecule design~\cite{koziarski2024rgfn,zhu2023sample}, structured data synthesis~\cite{nguyen2023hierarchical,jain2022biological}, combinatorial optimization~\cite{zhang2023let}, and causal discovery~\cite{manta2023gflownets,nguyen2023causal}. 
Unlike other generative models~\cite{goodfellow2014generative,kingma2022autoencodingvariationalbayes,song2020score}, GFlowNet produces samples by incrementally adding entities onto current structures subject to the requirement being proportional to a target reward distribution.
This compositional characteristic has drawn significant attention in both theoretical and applied machine learning communities, with early foundational works laying the groundwork for understanding their modeling capabilities. Specifically, the seminal paper ``GFlowNet Foundations''~\cite{bengio2023gflownet} provided rigorous insights into the principles governing the flow-based generative process, marking an essential milestone in the development and formalization of GFlowNets.

Recently, some theoretical trials in GFlowNets have seeked to establish a more rigorous foundation for this promising generative framework. Analyses of stability and expressiveness were conducted in \cite{silva2024analyzing} to show how flow imbalances propagate and architectural choices affect capability. \cite{pan2023better} improved the sample efficiency through local credit assignment and incomplete trajectories. \cite{silvagflownets} characterized conditions under which accurate distribution learning can be performed. A theoretical study was proposed for extending GFlowNets to continuous spaces \cite{lahlou2023theory}. Besides, \cite{krichel2024generalization} formalized the generalization guarantee of GFlowNets under certain conditions. Altogether, these studies have strengthened GFlowNets' theoretical underpinnings to some extent, enabling application to increasingly complex problems.

Despite the valuable contributions of the foundational studies, there remains a conspicuous gap in the in-depth and comprehensive understanding of GFlowNet’s learning behavior. 
For instance, critical aspects such as the convergence properties of GFlowNet algorithms, their sample complexity, and the robustness have seen limited investigation. Additionally, the intriguing phenomenon of implicit regularization -- often observed empirically but poorly understood in GFlowNets -- calls for deeper theoretical examination. A comprehensive understanding of these factors is crucial not only for advancing the theoretical foundations of GFlowNets but also for informing their practical deployment in real-world applications.

This paper takes a systematic approach to bridge this gap by presenting the first unified theoretical study of GFlowNets’ learning behavior. Our analysis focuses on uncovering fundamental insights relating to convergence guarantees, sample efficiency, implicit regularization mechanisms, and robustness of GFlowNets. By doing so, we aim to establish new theoretical paradigms for understanding how GFlowNets learn and perform under varying conditions, as well as provide actionable implications for their design and optimization. These contributions are essential for not only enriching the theory of GFlowNets but also enabling their broader adoption and utility across disciplines. 

In summary, this study not only advances the theoretical groundwork for GFlowNets but also opens a new line of inquiry into the interplay between flow-based generative modeling and efficient learning. By shedding light on the intricate mechanisms underpinning GFlowNets’ learning dynamics, we aim to enrich the landscape of probabilistic modeling, inspiring both theoretical explorations and innovative applications of this compelling framework.

\section{Preliminary}
GFlowNets represent a novel class of probabilistic generative models introduced by \cite{bengio2023gflownet,zhang2022generative}. GFlowNets are designed to sample from a distribution proportional to a given reward function by learning stochastic policies that construct objects through sequential decision-making.

The mathematical foundation of GFlowNets builds upon concepts from reinforcement learning, particularly Markov decision processes (MDPs) and flow networks~\cite{bengio2023gflownet,deleu2023generative,ford2015flows}. In GFlowNets, the generative process is modeled as a sequential construction of objects through a directed acyclic graph (DAG), where each node represents a partial construction and edges represent actions that extend the partial construction. The flow through this network is trained to be proportional to the reward associated with the terminal states.
Formally, we define the concepts of a GFlowNet.

\begin{definition}[Flow Network]
A GFlowNet is a structure defined on a Directed Acyclic Graph (DAG) $G=(S,E)$:
\begin{itemize}
    \item $S$ is the state space with initial state $s_0\in S$ and terminal state $S_T\subset S$.
    \item $E\subseteq S\times S$ is the directed edge set.
    \item $A(s)$ is the (finite) action space available at state $s$.
    \item $R:S_T\rightarrow\mathbb{R}^+$ is the reward function defined on terminal states.
    \item $Z$ is the overall flow (sometimes call ``partition'') such that $Z=F(s_0)=\sum_{s\in S_T}F(s)$, which also can be viewed as the flux conservation.
    \item $\mathcal{T}=\{\tau | \tau=(s_0,...s_n),\; s_n\in S_T\}$ is the set of all possible trajetories in a GFlowNet.
\end{itemize}
\end{definition}

\begin{definition}[Flow Function]
\label{def:flow_func}
    A flow function $F:E\rightarrow\mathbb{R}^+$ assigns a non-negative value to each edge such that:
    \begin{itemize}
        \item The outflow from $s_0$ equals the total reward: $\sum_{s':(s_0,s')\in E}F(s_0\rightarrow s')=Z=\sum_{s_T\in S_T}R(s_T)$
        \item For each non-terminal, non-initial state $s\in S\backslash (\{s_0\}\cup S_T)$, flow is conserved: $\sum_{s':(s',s)\in E}F(s'\rightarrow s)=\sum_{s':(s,s')\in E}F(s\rightarrow s')$.
    \end{itemize}
\end{definition}

For a single node $s$, we further define $F(s)=\sum_{s':(s',s)\in E}F(s'\rightarrow s)=\sum_{s':(s,s')\in E}F(s\rightarrow s')$ to be the overall flow passing this node.

One can note that the flow consistency constraints in Definition~\ref{def:flow_func} form a linear system wherein the variable are the flow values carried on each edge. Thus, we can characterize these as the following proposition.

\begin{proposition}[Linear system characterization]
\label{prop:linear}
    The flow matching constraint defines an underdetermined linear system $\mathbf{Af}=\mathbf{b}$, where
    \begin{itemize}
        \item $\mathbf{f}\in\mathbb{R}^{|E|}$ is the vector of flows on all edges.
        \item $\mathbf{A}\in\{-1,0,1\}^{|S|\times |E|}$ is the incident matrix of the graph.
        \item $\mathbf{b}\in\mathbb{R}^{|S|}$ is a vector with $b_{s_0}=Z$, $b_{s_T}=-R(s_T)$ for all $s_T\in S_T$, and $b_s=0$ otherwise.
    \end{itemize}
\end{proposition}
As $|E|>|S|$ in most cases, the linear system is underdetermined. Throughout this paper, our analysis is built upon three popular objectives for training GFlowNets presented in the following theorem, which can further be interpreted as regularized versions of linear system in Proposition~\ref{prop:linear}.

\begin{theorem}
    The following learning objectives for GFlowNets are equivalent to solving different regularized versions of the underdetermined linear system:
    \begin{itemize}
        \item Flow Matching (FM):
        \begin{equation}\label{eq:obj_fm}
        \mathcal{L}_{\text{FM}}(\theta) = \mathbb{E}_{s \sim \rho} \left[ \left( \sum_{s':(s',s) \in E} F_\theta(s' \rightarrow s) - \sum_{s':(s,s') \in E} F_\theta(s \rightarrow s') \right)^2 \right]
        \end{equation}
        \item Detailed Balance (DB):
        \begin{equation}\label{eq:obj_db}
        \mathcal{L}_{\text{DB}}(\theta) = \mathbb{E}_{\tau \sim P_B} \left[ \left( \frac{F_\theta(s \rightarrow s')}{F_\theta(s')P_B(s|s')} - 1 \right)^2 \right]
        \end{equation}
        \item Trajectory Balance (TB):
        \begin{equation}\label{eq:obj_tb}
        \mathcal{L}_{\text{TB}}(\theta) = \mathbb{E}_{\tau \sim P_F} \left[ \left( \frac{P_F(\tau)Z}{R(s_T)} - 1 \right)^2 \right]
        \end{equation}
    \end{itemize}
\end{theorem}

Note in both Eqs.~\ref{eq:obj_db} and \ref{eq:obj_tb}, the original flow has been reparameterized by introducing the forward flow assignment distribution $P_F(\cdot)$ and backward flow assignment distribution $P_B(\cdot)$~\cite{bengio2023gflownet}. These objectives represent different approaches to enforce flow consistency, with FM directly enforcing flow conservation at each state, DB enforcing local consistency in transition probabilities, and TB enforcing global consistency at the trajectory level. 

In implementation, $F_\theta$, $P_F$, $P_B$, and $Z$ can (partially) be parameterized using neural networks. In parituclar, the parameterization of $P_F$ and $P_B$ can be realized readily using $\mathrm{softmax}$ activation, wherein the flow value is parameterized using exponential functions. For these facts, one can refer to the implementation of GFlowNets at the GitHub repository\footnote{\url{https://github.com/alexhernandezgarcia/gflownet}}, which further lead to some unique characteristics in our analysis.

These preliminaries establish a fundamental understanding of GFlowNets.

\section{Main Theory}
Our theoretical analysis establishes fundamental properties of GFlowNets, providing rigorous characterizations of their convergence, generalization, and robustness. First, we demonstrate that convergence rates vary significantly across objectives, with FM achieving $\mathcal{O}(1/\sqrt{T})$ convergence while DB requires $\mathcal{O}(1/T^{1/3})$ due to higher variance in ratio-based gradients. Our generalization bounds reveal that sample complexity scales as $\mathcal{O}(|S|L\log(|S|/\delta)/\epsilon^2)$, highlighting the critical roles of state space size and trajectory length. Furthermore, we prove that GFlowNets with different training objectives induce distinct implicit regularization effects: Flow Matching promotes maximum entropy solutions, Detailed Balance enforces KL-divergence minimization between forward and backward processes, and Trajectory Balance favors path-length efficiency. Additionally, we quantify how noise in reward signals affects performance, demonstrating that robustness scales with $R_{min}^{-4}$, underlining the importance of minimum reward thresholds. These results not only advance our theoretical understanding of GFlowNets but also provide practical guidance for architecture selection, objective function choice, and hyperparameter tuning in applications.

\subsection{Convergence Analysis}\label{sec:converge}
Understanding the convergence properties of GFlowNets is critical for establishing their theoretical soundness and practical reliability. Specifically, convergence analysis focuses on whether and under what conditions the training dynamics of a GFlowNet ensure that it accurately learns the target distribution over structured objects. Unlike traditional generative approaches such as probabilistic graphical models or energy-based models, GFlowNets optimize flow consistency -- a unique objective that requires both local and global distributional alignment. This distinctive training paradigm raises intricate questions about convergence guarantees, including the role of stochasticity, the interplay between reward design and flow consistency, and how different objectives adapt to varying structural complexities of the generative task. 
Though some works empirically discussed the convergence such as \cite{madan2023learning,silva2024divergence,shen2023towards}, we still lack a theoretical and fundamental understanding.
In this section, we investigate these theoretical nuances, providing formal results on the conditions for convergence, characterizing the convergence rates under various assumptions, and identifying potential pitfalls. 

To this end, we provide formal convergence rate on FM and DB objectives.

\begin{theorem}[Convergence Rate for FM]
\label{thm:conv_FM}
    Under the exponential parameterization $F_\theta(s\rightarrow s')=\exp (W_\theta(s,s'))$ with $W_\theta$ being $l$-Lipschitz continuous and bounded gradients $\|\nabla_\theta \mathcal{L}_{FM}(\theta)\|\leq G$, descent with learning rate $\eta_t=\frac{\eta_0}{\sqrt{t}}$ achieves:
    \begin{equation}
        \mathbb{E}\left[\min_{t\in\{1,...,T\}} \|\nabla_\theta \mathcal{L}_{FM}(\theta)\|^2\right]\leq\frac{C}{\sqrt{T}}
    \end{equation}
    where $C=\frac{2(\mathcal{L}_{FM}(\theta_1)-\mathcal{L}_{FM}^*)}{\eta_0}+\frac{\eta_0 G^2}{2}$ and $\mathcal{L}_{FM}^*$ is the global minimum.
\end{theorem}

\begin{theorem}[Convergence Rate for DB]
    \label{thm:conv_DB}
    Under similar conditions as Theorem \ref{thm:conv_FM}, but with the DB objective, the convergence rate is bounded by $\mathcal{O}(1/T^{\frac{1}{3}})$ when using learning rate $\eta_t=\eta_0/t^{2/3}$.
\end{theorem}

Proofs of both theorems are in Appendix \ref{app:convergence}. Theorem \ref{thm:conv_FM} and \ref{thm:conv_DB} collectively suggest a more conservative learning rate schedule of DB compared to the standard schedule utilized for the FM. The slower convergence rate of $\mathcal{O}(1/T^{1/3})$ compared to $\mathcal{O}(1/\sqrt{T})$ reflects the inherent difficulty of optimizing ratio-based objectives with potentially high variance. We clarify that this analysis assumes a specific parameterization and that $P_B(s|s')$ is fixed. Therefore, different parameterizations or jointly learning $P_B$ could potentially improve the convergence properties.

Besides, Theorem \ref{thm:conv_DB} identifies that the key challenge in optimizing the DB objective stems from potential divisions by small backward probabilities, which can lead to high-variance gradient estimates requiring more cautious optimization strategies. This further poses a possible direction to improve the training of DB objective -- developing low-variance gradient estimator. While the design principles are out of the scope of this study, readers are referred to somee related works~\cite{mohamed2020monte,shi2022gradient,chen2019fast} for more details.

\subsection{Sample Complexity}\label{sec:sample_complexity}
The sample efficiency of GFlowNets plays a pivotal role in determining their practical utility, especially in resource-constrained scenarios where computational or data budgets are limited. In this section, we first seek to quantify the number of training trajectories or environment interactions required for a GFlowNet to accurately approximate the target distribution within a specified tolerance. 
Building upon the basic understanding of the sample complexity, we further investigate the impact of sample order, which potentially may help design more efficient active learning mechanism for GFlowNets.
We then study how long trajectories and the size of the state space matter, as well as the dynamics of propagating errors along with trajectories.
By identifying the underlying trade-offs and bottlenecks, this analysis builds a foundation for devising more sample-efficient training algorithms and offers a rigorous perspective on scaling GFlowNets to large and complex generative tasks. All proofs are presented in Appendix \ref{app:sample_complexity}.

\begin{theorem}[Sampling Impact on Convergence]
    \label{thm:sample_conv}
    Let $P_{sample}(\tau)$ be the probability of sampling trajectory $\tau$ and $P_{target}(\tau)\propto R(s_T)$ be the target distribution. Define the discrepency measure:
    \begin{equation}
        D(P_{sample},P_{target})=\max_{\tau\in\mathcal{T}}\frac{P_{target}(\tau)}{P_{sample}(\tau)}
    \end{equation}
    Then, for the TB objective, the number of samples $N$ required to achieve $\epsilon$-accuracy with probability $1-\delta$ is:
    \begin{equation}
        N=\mathcal{O}\left( \frac{D(P_{sample},P_{target})\cdot\log(1/\delta)}{\epsilon^2}\right)
    \end{equation}
\end{theorem}

This theorem demonstrates that the sample complexity depends critically on how well the sampling distribution covers trajectories that have high probability under the target distribution. This has several practical implications:
\begin{itemize}
    \item Sampling strategies that prioritize exploration of diverse trajectories can reduce the discrepancy measure and accelerate convergence.
    \item As training progresses and $P_F$ approaches $P_{target}$, using on-policy sampling becomes increasingly efficient.
    \item Adaptive sampling strategies that dynamically adjust $P_{sample}$ to minimize discrepancy can significantly reduce sample complexity.
    \item The discrepancy measure $D(P_{sample},P_{target})$ provides a formal metric for evaluating the quality of different sampling strategies in GFlowNet training.
\end{itemize}
In general, this theorem formalizes the intuition that proper coverage of the trajectory space is essential for efficient GFlowNet learning, particularly for the TB objective.

Next, we investigate how the order of samples impact the learning behavior of GFlowNets.

\begin{proposition}[Order-Dependent Convergence]
    \label{prop:order_convergence}
    For a GFlowNet trained with sequential sampling of trajectories $\{\tau_1,...,\tau_N\}$, the flow distribution after training satisfies:
    \begin{equation*}
        \|F_\theta-F^*\|_1\leq C\cdot\sum_{i=1}^N\alpha_i\cdot\|F_{\tau_i}-F^*\|_1
    \end{equation*}
    where $F^*$ is the optimal flow, $F_{\tau_i}$ is the flow along trajectory $\tau_i$, and $\{\alpha_i\}$ are path-dependent weights with $\sum_i\alpha_i=1$.
\end{proposition}

From Proposition \ref{prop:order_convergence}, we can conclude several insights:
\begin{itemize}
    \item The weights $\alpha_i$ depend on the sequence of trajectories, with earlier trajectories typically having less influence due to the cumulative effect of subsequent updates.
    \item The bound demonstrates why GFlowNet training exhibits path-dependent behavior – the final solution depends not just on which trajectories were sampled, but also on their order.
    \item The proposition formalizes how GFlowNets perform incremental learning, with each new trajectory modifying the current flow estimate.
\end{itemize}
This result helps explain observed phenomena in GFlowNet training, particularly why training dynamics can vary significantly based on the sequence of sampled trajectories, even when the same trajectories are used in different orders.

We next investigate how the training is related to the trajectory length and state size. Before going into the details, we make the following assumptions. It should be noted that these assumptions are quite mild in practice.
\begin{assumption}[State Visitation Distribution Assumption]
\label{assump:state_dist}
    Let $\pi_t(s)$ be the probability of state $s$ appearing in a trajectory sampled at training iteration $t$. We assume there exists a constant $c>0$ such that:
    \begin{equation*}
        \min_{s \in S, t \in \{1,2,...,T\}} \pi_t(s) \geq \frac{c}{|S|}
    \end{equation*}
\end{assumption}
\begin{assumption}[Error Independence Assumption]
\label{assump:err_ind}
    For a trajectory $\tau=(s_0,s_1,...,s_L)$, the estimation errors $\epsilon_i$ for flows on edge $(s_{i-1},s_i)$ satisfies:
    \begin{equation*}
        \text{Cov}(\varepsilon_i, \varepsilon_j) \leq \rho^{|i-j|} \cdot \sigma^2
    \end{equation*}
    for some constant $\rho\in[0,1)$ and variance $\sigma^2$. This bounds the correlation between errors at different steps of the trajectory.
\end{assumption}
\begin{assumption}[Information Content Assumption]
\label{assump:info_cont}
    Each trajectory $\tau$ of length $L$ provides $\Omega(L)$ independent constraints on the linear system $\mathbf{Af}=\mathbf{b}$. Formally, if $\mathbf{A}_\tau$ represents the submatrix of $\mathbf{A}$ corresponding to the constraints from trajectory $\tau$, then:
    \begin{equation}
        \mathrm{rank}(\mathbf{A}_\tau) = \Omega(L)
    \end{equation}
\end{assumption}
\begin{assumption}[Error Propagation Assumption]
\label{assump:err_prop}
    The error in estimating the flow along a trajectory $\tau$ of length $L$ scales as:
    \begin{equation}
        \|\hat{F}_\tau - F^*_\tau\|_2 = \mathcal{O}(\sqrt{L} \cdot \epsilon_{\text{edge}})
    \end{equation}
    where $\epsilon_\text{edge}$ is the average error in individual edge flow estimates.
\end{assumption}
Now that we present the following theorem.
\begin{theorem}[Sample-Trajectory Length Trade-off]
    \label{thm:traj_len_tradeoff}
    For a GFlowNet with maximum trajectory length $L$ and state space size $|S|$, to achieve an $\epsilon$-accurate flow distribution (in terms of total variation distance), the required sample complexity is:
    \begin{equation}
        N=\mathcal{O}\left( \frac{|S|L\log(|S|/\delta)}{\epsilon^2}\right)
    \end{equation}
\end{theorem}
This theorem captures the three key factors affecting GFlowNet sample complexity: 1. The size of the state space $|S|$; 2. The maximum trajectory length $|L|$; 3. The desired accuracy $\epsilon$ in the flow distribution. The logarithmic dependence on $1/\delta$ is standard in high-probability bounds, ensuring the result holds with probability at least $1-\delta$.

\begin{proposition}[Error Accumulation]
\label{prop:err_accu}
    For a trajectory $\tau=(s_0,s_1,...,s_L)$ of length $L$, the error in the flow estimation propagates as:
    \begin{equation}
        \mathbb{E}[|F_\theta(\tau)-F^*(\tau)|^2]\leq C\cdot(1+\gamma)^L\cdot\sum_{i=0}^{L-1}\mathbb{E}[|F_\theta(s_i\rightarrow s_{i+1})-F^*(s_i\rightarrow s_{i+1})|^2]
    \end{equation}
    where $\gamma>0$ and $C>0$ are constants, under the following assumptions:
    \begin{enumerate}
        \item There exist positive constants $m$ and $M$ such that $m<F_\theta(s)<M$ and $m<F^*(s)<M$ for all $s\in S$.
        \item The transition probabilities are bounded away from zero: $P_{F_\theta}(s_{i+1}|s_i)\geq\epsilon>0$ and $P_{F^*}(s_{i+1}|s_i)\geq\epsilon>0$.
    \end{enumerate}
\end{proposition}

This proposition quantifies how errors accumulate along trajectories, explaining why longer trajectories are more challenging to learn accurately. The assumptions about bounded flows and transition probabilities are reasonable in practical GFlowNet implementations, especially when using common parameterizations like softmax for transition probabilities. This result has important implications for GFlowNet training strategies, suggesting that: 1. Shorter trajectories will generally be easier to learn accurately; 2. Error control mechanisms become increasingly important as trajectory length increases; 3. The sample complexity must scale with trajectory length to maintain a fixed error bound.

\subsection{Implicit Regularization}\label{sec:im_reg}
Implicit regularization is a fundamental aspect of training dynamics in machine learning models, often shaping their generalization performance and optimization outcomes without explicit penalties in the objective. In the context of GFlowNets, implicit regularization emerges prominently through the choice of training objectives, particularly the widely-adopted Trajectory Balance (TB) and Detailed Balance (DB) objectives. These objectives not only guide the learning process towards distributional consistency but also impose distinct forms of regularization on the learned policies. In this section, we theoretically dissect the implicit regularization effects of TB and DB, revealing their connections to maximum entropy and Kullback–Leibler (KL) divergence regularization, respectively. Specifically, we demonstrate how the TB objective inherently encourages max-entropy policies, promoting exploration and distributional smoothness, while the DB objective aligns closely with KL-regularized optimization, favoring precise adherence to the target distribution. By formalizing these observations and analyzing their implications, we aim to uncover the often-unspoken yet crucial role of implicit regularization in determining the expressiveness, robustness, and generalization capabilities of GFlowNet policies. All proofs are in Appendix \ref{app:imp_reg}.

\begin{theorem}[Implicit Max-Entropic Regularization of FM]
\label{thm:FM_implicit_reg}
    The Flow Matching (FM) objective, when parameterized as $F_\theta(s\rightarrow s')=\exp(W_\theta(s,s')$, implicitly maximizes the entropy of the flow distribution subject to flow constraints.
\end{theorem}

Theorem \ref{thm:FM_implicit_reg} establishes that GFlowNets with exponential parameterization and the FM objective are performing maximum entropy inference, which has connections to principles in statistical physics and information theory. In this mechanism, maximum entropy acts as an implicit regularizer, preventing the flow from concentrating too heavily on a few paths and encouraging exploration of diverse trajectories. This also supports the observation that GFlowNet can produce diverse samples. Furthermore, we can conclude the implicit regularization mechanism with DB objective as follows.

\begin{proposition}[Implicit KL-Regularization of DB]
\label{prop:impli_reg_DB}
    The Detailed Balance (DB) objective for GFlowNets induces an implicit regularization effect. Specifically, when the forward and backward processes are close to satisfying detailed balance, minimizing the DB objective approximates minimizing the KL-divergence between the joint state-transition distributions of the forward and backward processes.
\end{proposition}

In summary, the implicit regularization effects induced by the Trajectory Balance (TB) and Detailed Balance (DB) objectives play a pivotal role in shaping the learning dynamics and generalization properties of GFlowNets. Our analysis highlights the distinct nature of these objectives: while TB promotes maximum entropy regularization that encourages exploration and robustness to uncertain or noisy reward landscapes, DB imposes KL-regularization that favors precision and tighter adherence to the target distribution. These regularization effects not only influence the optimal policy learned by GFlowNets but also impact their behavior in scenarios with sparse rewards or complex generative tasks. Importantly, understanding these implicit biases provides practical guidelines for selecting the appropriate objective based on the structure of the problem and the desired trade-offs between exploration and exploitation. As GFlowNets continue to gain traction in applications, a more nuanced understanding of the interplay between objectives, implicit regularization, and task-specific requirements will be instrumental in their effective deployment and further development.

\subsection{Robustness}\label{sec:robustness}
Robustness is a critical property for any generative modeling framework, particularly in real-world scenarios where noise and imperfections in the data, reward signals, or optimization process are inevitable. For GFlowNets, robustness directly affects their ability to maintain reliable performance and approximate the target distribution accurately in the presence of small perturbations. In this section, we theoretically examine the robustness of GFlowNets by analyzing the objective error and the resulting distribution drift under small noise in the reward function, transition dynamics, or policy updates. Additionally, we explore the implications of these perturbations on sample complexity, identifying how noise levels impact the number of trajectories required to achieve an accurate distributional approximation. By quantifying the sensitivity of GFlowNet objectives to noise and characterizing the trade-offs between robustness and sample efficiency, this analysis aims to deepen our understanding of the stability and resilience of GFlowNets. 

\begin{theorem}[Objective Error Under Reward Noise]
\label{thm:obj_err_noise}
    Let $\mathcal{F}_\theta$ be a GFlowNet trained with the Trajectory Balance (TB) objective on rewards $R(s_T)$. If we instead use noisy rewards $\tilde{R}(s_T)=R(s_T)+\varepsilon(s_T)$ where $\varepsilon(s_T)$ represents a noise on dependent on the terminal state $s_T$, then the expected increase in the TB loss is bounded by:
    \begin{equation}
        \mathbb{E}[\mathcal{L}_{TB}(\mathcal{F}_\theta,\tilde{R})-\mathcal{L}_{TB}(\mathcal{F}_\theta,R)]\leq \frac{Z^2}{R_{min}^4}\cdot\mathbb{E}\left[\varepsilon(s_T)^2\right]
    \end{equation}
    where $R_{min}>0$ is the minimal reward value, and $P_(s_T)$ is the probability of the terminal state $s_T$.
\end{theorem}

We can draw several actionable insights from Theorem \ref{thm:obj_err_noise}: 1. Implement reward thresholding by adding a small positive constant to all rewards can effectively increase $R_{min}$ and dramatically reduce sensitivity to noise; 2. When scaling rewards, multiplicative scaling is preferable; 3. As the bound contains $Z^2$ in the numerator, we can consider normalizing rewards to keep $Z$ bounded while preserving relative reward values; 4. Since the theorem suggests that states with smaller rewards are more sensitive to noise, we can implement reward-aware learning rate schedules that adjust based on the magnitude of rewards encountered during training, reducing learning rates when processing transitions leading to low-reward states.

\begin{theorem}[Distribution Drift Under Zero-Mean Uniform Noise]
\label{thm:dis_drift_zero_mean}
    Let $\mathcal{F}_\theta$ be a GFlowNet trained to convergence with the TB objective on true rewards $R(s_T)$. If we instead use noisy rewards $\tilde{R}(s_T)=R(s_T)+\varepsilon(s_T)$ where $\varepsilon(s_T)$ is i.i.d. small zero-mean noise with uniform variance $\sigma^2$ across all terminal states, then the expected KL-divergence between the resulting terminal state distributions is bounded by:
    \begin{equation}
        \mathbb{E}[D_{KL}(P_{\tilde{R}}(s_T)\|P_R(s_T))]\leq\frac{\sigma^2}{2}\left( \frac{1}{R_{min}^2}+\frac{|S_T|}{Z_R^2}\right)
    \end{equation}
    where $P_R(s_T)$ and $P_{\tilde{R}}(s_T)$ are the terminal state distribution under true and noisy rewards, respectively. $Z_R=\sum_{s_T}R(s_T)$ is the sum of all rewards.
\end{theorem}

This bound provides several insights for GFlowNet training with noisy rewards. First, the bound scales with $1/R_{min}^2$, suggesting that larger minimum rewards significantly improve robustness to noise. Second, The bound's second term scales with the number of terminal states $|S_T|$, indicating that larger state spaces amplify the impact of noise. Third, The bound's second term decreases with the squared sum of rewards $Z_R^2$, showing that larger overall reward magnitudes can help counteract the state space size effect. Last but not least, the linear scale-up with variance $\sigma^2$ confirms that reducing measurement noise directly improves distribution robustness. The bound suggests that scaling all rewards uniformly (which increases both $R_{min}$ and $Z_R$ proportionally) would improve robustness in practice.

\begin{theorem}[Sample Complexity with Noisy Rewards]
\label{thm:sample_comp_noisy}
    Let $N(\epsilon,\delta,R)$ be the number of samples required to achieve $\epsilon$-accuracy in flow values with probability $1-\delta$ under true rewards $R$, and let $N(\epsilon,\delta,\tilde{R})$ be the corresponding number under rewards with small noise $\tilde{R}(s_T)=R(s_T)+\varepsilon(s_T)$ with variance $\sigma^2$. Then:
    \begin{equation}
        \frac{N(\epsilon,\delta,\tilde{R})}{N(\epsilon,\delta,R)}\leq 1+\frac{CZ^2\sigma^2}{\epsilon^2R^4_{min}}
    \end{equation}
    where $C>0$ is a constant.
\end{theorem}

This theorem provides several key insights for GFlowNet training under noisy rewards. First, the sample complexity increases linearly with noise variance $\sigma^2$. Second, Larger total flows (higher $Z$) amplify the impact of noise, suggesting that normalizing rewards might improve robustness. Besides, The impact of noise decreases quadratically with larger accuracy targets $\epsilon$, meaning that noise is more problematic when high precision is required. Last, The sample complexity increase is inversely proportional to $R_{min}^4$, highlighting the critical importance of ensuring rewards are not too small. These insights provide practical guidance for designing robust GFlowNet training procedures when working with noisy reward estimates, suggesting strategies such as reward scaling, minimum reward thresholding, and adaptive accuracy targets based on noise estimates.

In conclusion, this section highlights the importance of robustness in GFlowNets, particularly in the face of small noise that can arise in practical scenarios. We have demonstrated how noise in rewards or dynamics can propagate through the training process, leading to objective errors and distributional drift, and further examined how these factors impact the sample complexity required for accurate learning.

\section{Limitations and Future Work}\label{sec:limitation}
While our study provides significant theoretical insights into the learning behavior of GFlowNets across dimensions such as convergence, sample complexity, implicit regularization, and robustness, several limitations remain as follows. 
\begin{itemize}
    \item \textbf{Parametric Assumptions}: Most of our analyses assume specific parametrizations (e.g., exponential parameterization for flows) or well-behaved reward structures. In practice, GFlowNets may use diverse parametric forms (particularly neural networks) and encounter more complex reward landscapes. Our results on implicit regularization particularly depend on specific neural network architectures, potentially limiting their applicability to novel architectural choices.
    \item \textbf{Asymptotic Nature of Bounds}: Many of our convergence and sample complexity bounds are asymptotic, characterizing behavior as the number of samples or iterations approaches infinity. These asymptotic guarantees may not fully capture the practical finite-sample behavior that practitioners encounter, especially during early stages of training.
    \item \textbf{Independence Assumptions}: Several proofs rely on assumptions of independent noise or independent sampling of trajectories. In real applications, correlations often exist between rewards of similar states or between consecutive samples in training procedures, potentially affecting the tightness of our bounds.
    \item \textbf{Limited Scope of Objectives}: While we analyzed the three primary GFlowNet objectives (FM, DB, and TB), many hybrid objectives and variants exist in practice. The interactions between these objectives when used in combination remain theoretically underexplored, despite their empirical success.
\end{itemize}
All the aforementioned limitations posit directions in our future work, which as we acknowledge are challenging but meaningful. Some of these unique challenges are:
\begin{itemize}
    \item \textbf{Non-Asymptotic Bounds}: Developing non-asymptotic bounds that characterize finite-sample behavior would provide more practical guidance for GFlowNet implementations. Specifically, understanding how quickly convergence occurs in practice based on problem characteristics would help in setting appropriate iteration counts and convergence criteria.
    \item \textbf{Adaptive Training Procedures}: Building on our noise robustness results, future work should develop theoretically-grounded adaptive training procedures that automatically adjust learning rates, reward scaling, and regularization based on estimated noise levels and reward distributions.
    \item \textbf{Function Approximation Guarantees}: Extending convergence guarantees to more general function approximation settings, particularly for deep neural networks with realistic architectures, would bridge the gap between theory and practice. This includes analyzing the interplay between network depth, width, and GFlowNet performance.
    \item \textbf{Continuous State Spaces}: Though some trials exist~\cite{lahlou2023theory}, extending our theoretical results to continuous state spaces would broaden the applicability of GFlowNets to domains like molecular design, materials discovery, and continuous control. This requires addressing the challenges of infinite state spaces and functional optimization.
\end{itemize}

\section{Conclusion}
In this paper, we presented a systematic theoretical study of Generative Flow Networks (GFlowNets), focusing on key aspects of their learning behavior, including convergence, sample complexity, implicit regularization mechanisms, and robustness against noise. Our findings provide formal guarantees for convergence and sample efficiency, illuminate the natural emergence of implicit structural regularization, and offer actionable insights for optimizing GFlowNet design and training. By bridging critical gaps in the understanding of their theoretical underpinnings, this work lays the foundation for deeper study and broader application of GFlowNets in diverse domains. While challenges remain in scaling and integrating GFlowNets with other paradigms, our contributions open exciting avenues for future research, advancing both theory and practice in this promising field.

\bibliography{reference}
\bibliographystyle{plain}
\clearpage

\appendix
\section{Proofs of Convergence}\label{app:convergence}
\begin{proof}[Proof of Theorem \ref{thm:conv_FM}]
    For SGD updates $\theta_{t+1}=\theta_t-\eta_t g_t$ where $g_t$ is the stochastic gradient at iteration $t$, we have by $l$-smoothness:
    \begin{equation}
        \mathcal{L}_{FM}(\theta_{t+1})\leq\mathcal{L}_{FM}(\theta_t)-\eta_t\langle \nabla_\theta\mathcal{L}_{FM}(\theta_t),g_t\rangle+\frac{l\eta_t^2\|g_t\|^2}{2}
    \end{equation}
    Taking expectations conditioned on $\theta_t$ and using $\mathbb{E}[g_t|\theta_t]=\nabla_\theta\mathcal{L}_{FM}(\theta_t)$:
    \begin{equation}
        \mathbb{E}[\mathcal{L}_{FM}(\theta_{t+1})|\theta_t]\leq\mathcal{L}_{FM}(\theta_t)-\eta_t\|\mathcal{L}_{FM}(\theta_t)\|^2+\frac{l\eta_t^2\mathbb{E}[\|g_t\|^2|\theta_t]}{2}
    \end{equation}
    Using the bounded gradient assumption, $\mathbb{E}[\|g_t\|^2|\theta_t]\leq G^2$,
    \begin{equation}
        \mathbb{E}[\mathcal{L}_{FM}(\theta_{t+1})|\theta_t]\leq\mathcal{L}_{FM}(\theta_t)-\eta_t\|\mathcal{L}_{FM}(\theta_t)\|^2+\frac{l\eta_t^2G^2}{2}
    \end{equation}
    Rearranging to isolate the gradient norm:
    \begin{equation}
        \eta_t\|\mathcal{L}_{FM}(\theta_t)\|^2\leq\mathcal{L}_{FM}(\theta_{t})-\mathbb{E}[\mathcal{L}_{FM}(\theta_{t+1})]+\frac{l\eta_t^2G^2}{2}
    \end{equation}
    Taking full expectation:
    \begin{equation}
        \eta_t\mathbb{E}[\|\mathcal{L}_{FM}(\theta_t)\|^2]\leq\mathbb{E}[\mathcal{L}_{FM}(\theta_{t})]-\mathbb{E}[\mathcal{L}_{FM}(\theta_{t+1})]+\frac{l\eta_t^2G^2}{2}
    \end{equation}
    Summing from $t=1$ to $T$:
    \begin{equation}
        \sum_{t=1}^T\eta_t\mathbb{E}[\|\mathcal{L}_{FM}(\theta_t)\|^2]\leq\mathbb{E}[\mathcal{L}_{FM}(\theta_{1})]-\mathbb{E}[\mathcal{L}_{FM}(\theta_{T+1})]+\frac{lG^2}{2}\sum_{t=1}^T\eta_t^2
    \end{equation}
    Since $\mathcal{L}_{FM}(\theta_{T+1})\geq\mathcal{L}^*_{FM}$ and $\eta_t=\frac{\eta_0}{\sqrt{t}}$, we have:
    \begin{equation}
        \sum_{t=1}^T\eta_t\mathbb{E}[\|\mathcal{L}_{FM}(\theta_t)\|^2]\leq\mathbb{E}[\mathcal{L}_{FM}(\theta_{1})]-\mathcal{L}^*_{FM}+\frac{lG^2\eta_0^2}{2}\sum_{t=1}^T\frac{1}{t}
    \end{equation}
    To derive the minimum gradient norm bound, we observe:
    \begin{equation}
        \min_{t\in\{1,...,T\}}\mathbb{E}[\|\mathcal{L}_{FM}(\theta_t)\|^2]\leq\frac{\sum_{t=1}^T\eta_t\|\mathcal{L}_{FM}(\theta_t)\|^2}{\sum_{t=1}^T \eta_t}
    \end{equation}
    This inequality holds because the minimum cannot exceed the weighted average when all weights $\eta_t$ are positive.

    As we have:
    \begin{itemize}
        \item $\sum_{t=1}^T\frac{1}{t}\leq 1+\ln (T)$
        \item $\sum_{t=1}^T\eta_t=\eta_0\sum_{t=1}^T\frac{1}{\sqrt{t}}\geq\eta_0\int_1^{T+1}\frac{dx}{\sqrt{x}}\geq 2\eta_0(\sqrt{T}-1)$
    \end{itemize}
    Therefore,
    \begin{equation}
        \min_{t\in\{1,...,T\}}\mathbb{E}[\|\mathcal{L}_{FM}(\theta_t)\|^2]\leq\frac{\mathbb{E}[\mathcal{L}_{FM}(\theta_1)]-\mathcal{L}^*_{FM}+\frac{lG^2\eta_0^2}{2}(1+\ln (T))}{2\eta_0(\sqrt{T}-1)}
    \end{equation}
    As $T\rightarrow\infty$, we have $\frac{T}{\sqrt{T}}\rightarrow 0$, this simplifies to:
    \begin{equation}
        \min_{t\in\{1,...,T\}}\mathbb{E}[\|\mathcal{L}_{FM}(\theta_t)\|^2]\leq\frac{\mathbb{E}[\mathcal{L}_{FM}(\theta_1)]-\mathcal{L}^*_{FM}+\frac{lG^2\eta_0^2\ln(T)}{2}}{2\eta_0\sqrt{T}}=\frac{C}{\sqrt{T}}
    \end{equation}
    where $C=\frac{2(\mathcal{L}_{FM}(\theta_1)-\mathcal{L}_{FM}^*)}{\eta_0}+\frac{\eta_0 G^2}{2}$. 
    By Jensen's inequality, since the minimum operation is convex, we have:
    \begin{equation}
        \mathbb{E}[\min_{t\in\{1,...,T\}}\|\mathcal{L}_{FM}(\theta_t)\|^2]\leq \min_{t\in\{1,...,T\}}\mathbb{E}[\|\mathcal{L}_{FM}(\theta_t)\|^2]
    \end{equation}
    This completes the proof.
\end{proof}

\begin{proof}[Proof of Theorem \ref{thm:conv_DB}]
    We first characterize the gradient and its properties.
    For a local transition $(s,s')$, let us define its loss as
    \begin{equation}
        l_\theta(s,s')=\left( \frac{F_\theta(s\rightarrow s')}{F_\theta(s')P_B(s|s')}-1\right)^2
    \end{equation}
    The gradient of this loss w.r.t. $\theta$ is:
    \begin{equation}
        \nabla_\theta l_\theta(s,s')=2\left( \frac{F_\theta(s\rightarrow s')}{F_\theta(s')P_B(s|s')}-1\right)\cdot\nabla_\theta\left( \frac{F_\theta(s\rightarrow s')}{F_\theta(s')P_B(s|s')}\right)
    \end{equation}
    Given the exponential parameterization $F_\theta(s\rightarrow s')=\exp(W_\theta(s,s'))$, we can compute
    \begin{equation}
        \nabla_\theta\left( \frac{F_\theta(s\rightarrow s')}{F_\theta(s')P_B(s|s')}\right)=\frac{F_\theta(s\rightarrow s')}{F_\theta(s')P_B(s|s')}\cdot\left( \nabla_\theta W_\theta(s,s')-\frac{\sum_{s'':(s',s'')\in E}F_\theta(s'\rightarrow s'')\nabla_\theta W_\theta(s',s'')}{F_\theta(s')}\right)
    \end{equation}
    Next, we are going to establish the variance bound. The key challenge in establishing the convergence rate for the DB objective arises from the potentially high variance in the gradients, particularly when $P_B(s|s')$ is small.

    For any sampling distribution $\rho_{DB}$ over transitions, we analyze the second moment of the gradient:
    \begin{equation}
        \mathbb{E}_{(s,s')\sim\rho_{DB}}\left[ \|\nabla_\theta l_\theta(s,s')\|^2\right]
    \end{equation}
    Using our derivation above and applying Cauchy-Schwarz inequality:
    \begin{equation}
        \|\nabla_\theta l_\theta(s,s')\|^2\leq 4\left( \frac{F_\theta(s\rightarrow s')}{F_\theta(s')P_B(s|s')}-1\right)^2\cdot\left( \frac{F_\theta(s\rightarrow s')}{F_\theta(s')P_B(s|s')}\right)^2\cdot M^2
    \end{equation}
    This bound exists due to the $l$-Lipschitz continuity assumption on $W_\theta$.

    Let $R(s,s',\theta)=\frac{F_\theta(s\rightarrow s')}{F_\theta(s')P_B(s|s')}$. We then have:
    \begin{equation}
        \|\nabla_\theta l_\theta(s,s')\|^2\leq(R(s,s',\theta)-1)^2\cdot R(s,s',\theta)^2\cdot M^2
    \end{equation}
    The variance term can grow significantly when $P_B(s|s')$ is small. To formalize this, we define:
    \begin{equation}
        K(\theta)=\sup_{(s,s')\in E}\frac{1}{P_B(s|s')^2}
    \end{equation}
    Then
    \begin{equation}
        \mathbb{E}_{(s,s')\sim\rho_{DB}}\left[ \|\nabla_\theta l_\theta(s,s')\|^2\right]\leq 4M^2\cdot K(\theta)\cdot\mathbb{E}_{(s,s')\sim\rho_{DB}}[(R(s,s',\theta)-1)^2\cdot F_\theta(s\rightarrow s')^2]
    \end{equation}
    To ensure convergence with this high variance, we need a more conservative learning rate schedule.

    Let $g_t=\nabla_\theta l_\theta(s_t,s_t')$ be the stochastic gradient at iteration $t$, where $(s_t,s_t')$ is sampled from $\rho_{DB}$.

    For SGD updates $\theta_{t+1}=\theta_t-\eta_t g_t$ with learning rate $\eta_t=\eta_0/t^{\frac{2}{3}}$, we analyze:
    \begin{equation}
        \mathbb{E}[\mathcal{L}_{DB}(\theta_{t+1})]\leq\mathcal{L}_{DB}(\theta_t)-\eta_t\mathbb{E}[\langle \nabla_\theta\mathcal{L}_{DB}(\theta_t),g_t\rangle]+\frac{l\eta_t^2}{2}\mathbb{E}[\|g_t\|^2]
    \end{equation}
    Since $\mathbb{E}[g_t]=\nabla_\theta\mathcal{L}_{DB}(\theta_t)$, we have
    \begin{equation}
        \mathbb{E}[\mathcal{L}_{DB}(\theta_{t+1})]\leq\mathcal{L}_{DB}(\theta_t)-\eta_t\|\nabla_\theta\mathcal{L}_{DB}(\theta_t)\|^2+\frac{l\eta_t^2}{2}\mathbb{E}[\|g_t\|^2]
    \end{equation}
    From our variance bound, let $\sigma^2_t=\mathbb{E}[\|g_t\|^2]$, which satisfies:
    \begin{equation}
        \sigma^2_t\leq C\cdot K(\theta_t)\cdot\mathcal{L}_{DB}(\theta_t)
    \end{equation}
    for some constant $C>0$.

    Then we derive the convergence rate of DB. Taking full expectation:
    \begin{equation}
        \mathbb{E}[\mathcal{L}_{DB}(\theta_{t+1})]\leq\mathbb{E}[\mathcal{L}_{DB}(\theta_t)]-\eta_t\mathbb{E}[\|\nabla_\theta\mathcal{L}_{DB}(\theta_t)\|^2]+\frac{l\eta_t^2}{2}\mathbb{E}[\|g_t\|^2]
    \end{equation}
    Substituting the variance bound:
    \begin{equation}
        \mathbb{E}[\mathcal{L}_{DB}(\theta_{t+1})]\leq\mathbb{E}[\mathcal{L}_{DB}(\theta_t)]-\eta_t\mathbb{E}[\|\nabla_\theta\mathcal{L}_{DB}(\theta_t)\|^2]+\frac{l\eta_\theta^2 C}{2}\mathbb{E}[K(\theta_t)\cdot\mathcal{L}_{DB}(\theta_t)]
    \end{equation}
    Let us denote $K_\text{max}=\max_{t\leq T}\mathbb{E}[K(\theta_t)]$, which is finite under our assumptions about bounded $P_B(s|s')$. Then,
    \begin{equation}
        \mathbb{E}[\mathcal{L}_{DB}(\theta_{t+1})]\leq\mathbb{E}[\mathcal{L}_{DB}(\theta_t)]-\eta_t\mathbb{E}[\|\nabla_\theta\mathcal{L}_{DB}(\theta_t)\|^2]+\frac{l\eta_\theta^2 CK_\text{max}}{2}\mathbb{E}[\mathcal{L}_{DB}(\theta_t)]
    \end{equation}
    Rearraning:
    \begin{equation}
        \eta_t\mathbb{E}[\|\nabla_\theta\mathcal{L}_{DB}(\theta_t)\|^2]\leq\mathbb{E}[\mathcal{L}_{DB}(\theta_t)]-\mathbb{E}[\mathcal{L}_{DB}(\theta_{t+1})]+\frac{l\eta_t^2 CK_\text{max}}{2}\mathbb{E}[\mathcal{L}_{DB}(\theta_t)]
    \end{equation}
    Summing over $t=1$ to $T$:
    \begin{equation}
        \sum_{t=1}^T \eta_t\mathbb{E}[\|\nabla_\theta\mathcal{L}_{DB}(\theta_t)\|^2]\leq\mathbb{E}[\mathcal{L}_{DB}(\theta_1)]-\mathbb{E}[\mathcal{L}_{DB}(\theta_{T+1})]+\frac{l CK_\text{max}}{2}\sum_{t=1}^T \eta_t^2\mathbb{E}[\mathcal{L}_{DB}(\theta_t)]
    \end{equation}
    Since $\mathcal{L}_{DB}(\theta_{T+1})\geq 0$ and assuming $\mathbb{E}[\mathcal{L}_{DB}(\theta_t)]\leq B$ for all $t$ (bounded loss), we get:
    \begin{equation}
        \sum_{t=1}^T \eta_t\mathbb{E}[\|\nabla_\theta\mathcal{L}_{DB}(\theta_t)\|^2]\leq\mathbb{E}[\mathcal{L}_{DB}(\theta_1)]+\frac{l CK_\text{max}B}{2}\sum_{t=1}^T \eta_t^2
    \end{equation}
    With $\eta_t=\eta_0/t^{\frac{2}{3}}$, we have $\sum_{t=1}^T\eta_t\geq\eta_0\int_1^{T+1}\frac{dx}{x^{2/3}}=3\eta_0(T^{1/3}-1)$ and $\sum_{t=1}^T\eta_t^2\leq\eta_0^2\sum_{t=1}^T\frac{1}{t^{4/3}}\leq\eta_0^2(1+\int_1^T\frac{dx}{x^{4/3}})=\eta_0^2(1+3(1-T^{-1/3}))$. Substituting
    \begin{equation}
        3\eta_0(T^{1/3}-1)\min_{t\in\{1,...,T\}}\mathbb{E}[\|\nabla_\theta\mathcal{L}_{DB}(\theta_t)\|^2]\leq\mathbb{E}[\mathcal{L}_{DB}(\theta_1)]+\frac{3l CK_\text{max}B\eta_0^2}{2}(1+3(1-T^{-1/3}))
    \end{equation}
    Simplifying:
    \begin{equation}
        \min_{t\in\{1,...,T\}}\mathbb{E}[\|\nabla_\theta\mathcal{L}_{DB}(\theta_t)\|^2]\leq\frac{\mathbb{E}[\mathcal{L}_{DB}(\theta_1)]+\frac{3l CK_\text{max}B\eta_0^2}{2}(1+3(1-T^{-1/3}))}{3\eta_0(T^{1/3}-1)}
    \end{equation}
    As $T\rightarrow\infty$, this simplies to:
    \begin{equation}
        \min_{t\in\{1,...,T\}}\mathbb{E}[\|\nabla_\theta\mathcal{L}_{DB}(\theta_t)\|^2]\leq\frac{\mathbb{E}[\mathcal{L}_{DB}(\theta_1)]+\frac{6l CK_\text{max}B\eta_0^2}{2}}{3\eta_0T^{1/3}}=\mathcal{O}(1/T^{1/3})
    \end{equation}
    which establishes the $\mathcal{O}(1/T^{1/3})$ convergence rate for the DB objective.
\end{proof}

\section{Proofs of Sample Complexity}\label{app:sample_complexity}
\begin{proof}[Proof of Theorem \ref{thm:sample_conv}]
    In the proof Theorem \ref{thm:sample_conv}, we first prove a general bound of convergence using Hoeffding's inequality. Then we further present a tighter bound via Bernstein's inequality.

    \textbf{A General Bound}:
    
    We begin by formulating the gradient estimation of the TB objective as an importance sampling problem. The TB objective is defined as:
    \begin{equation}
        \mathcal{L}_{TB}(\theta)=\mathbb{E}_{\tau\sim P_F}\left[\left(\frac{P_F(\tau)Z}{R(s_T)}-1\right)^2\right]
    \end{equation}
    where $P_F$ is the forward sampling distribution induced by the current policy, $Z$ is the partition function, and $R(s_T)$ is the reward at the terminal state of trajectory $\tau$.

    When estimating this objective using a sampling distribution $P_{sample}(\tau)$ that may differ from $P_F(\tau)$, we employ importance sampling:
    \begin{equation}
        \mathcal{L}_{TB}(\theta)=\mathbb{E}_{\tau\sim P_{sample}}\left[ \frac{P_F(\tau)}{P_{sample}(\tau)}\left( \frac{P_F(\tau)Z}{R(s_T)}-1\right)^2\right]
    \end{equation}
    Consequently, the gradient of this objective is:
    \begin{equation}
        \nabla_\theta\mathcal{L}_{TB}(\theta)=\mathbb{E}_{\tau\sim P_{sample}}\left[ \frac{P_{target}(\tau)}{P_{sample}(\tau)}\cdot g_\theta(\tau)\right]
    \end{equation}
    where $g_\theta(\tau)$ represents the per-trajectory gradient contribution, and we have used the fact that $P_{target}(\tau)\propto R(s_T)$ to simplify notations.

    We then analyze the estimator variance. Given $N$ trajectories $\{\tau_1,...,\tau_N\}$ sampled independently from $P_{sample}(\tau)$, we construct an estimator for the gradient:
    \begin{equation}
        \hat{G}_N=\frac{1}{N}\sum_{i=1}^N\frac{P_{target}(\tau_i)}{P_{sample}(\tau_i)}\cdot g_\theta(\tau_i)
    \end{equation}
    By standard results from importance sampling theory, the variance of this estimator is:
    \begin{equation}
        \mathrm{Var}[\hat{G}_N]=\frac{1}{N}\mathrm{Var}_{\tau\sim P_{sample}}\left[\frac{P_{target}(\tau_i)}{P_{sample}(\tau_i)}\cdot g_\theta(\tau_i)\right]
    \end{equation}
    Expanding thie variance, we have:
    \begin{equation}
        \mathrm{Var}[\hat{G}_N]=\frac{1}{N}\left( \mathbb{E}_{\tau\sim P_{sample}}\left[\left( \frac{P_{target}(\tau_i)}{P_{sample}(\tau)}\cdot g_\theta(\tau)\right)^2\right] - \left(\mathbb{E}_{\tau\sim P_{sample}} \left[\frac{P_{target}(\tau)}{P_{sample}(\tau)}\cdot g_\theta(\tau)\right]\right)^2\right)
    \end{equation}
    Let us assume that $\|g_\theta(\tau)\|\leq G$ for all trajectories $\tau$, where $G$ is a positive constant. Then,
    \begin{equation}
        \mathbb{E}_{\tau\sim P_{sample}}\left[\left( \frac{P_{target}(\tau_i)}{P_{sample}(\tau)}\cdot g_\theta(\tau)\right)^2\right]\leq G^2\cdot \mathbb{E}_{\tau\sim P_{sample}}\left[\left( \frac{P_{target}(\tau)}{P_{sample}(\tau)}\right)^2\right]
    \end{equation}
    By definition of the discrepancy measure $D(P_{sample},P_{target})$:
    \begin{equation}
        \frac{P_{target}(\tau)}{P_{sample}(\tau)}\leq D(P_{sample},P_{target})\qquad\forall \tau\in\mathcal{T}
    \end{equation}
    Therefore:
    \begin{equation}
        \mathbb{E}_{\tau\sim P_{sample}}\left[\left( \frac{P_{target}(\tau_i)}{P_{sample}(\tau)}\right)^2\right]\leq D(P_{sample},P_{target})\cdot \mathbb{E}_{\tau\sim P_{sample}}\left[ \frac{P_{target}(\tau)}{P_{sample}(\tau)}\right]
    \end{equation}
    Since $\mathbb{E}_{\tau\sim P_{sample}}\left[ \frac{P_{target}(\tau)}{P_{sample}(\tau)}\right]=1$ (as both $P_{target}$ and $P_{sample}$ are properly normalized), we have
    \begin{equation}
        \mathbb{E}_{\tau\sim P_{sample}}\left[\left( \frac{P_{target}(\tau_i)}{P_{sample}(\tau)}\right)^2\right]\leq D(P_{sample},P_{target})
    \end{equation}
    Consequently,
    \begin{equation}
        \mathrm{Var}[\hat{G}_N]\leq\frac{G^2\cdot D(P_{sample},P_{target})}{N}
    \end{equation}
    To establish the sample complexity, we employ Hoeffding's inequality~\cite{hoeffding1994probability} for bounded random variables. Let's define:
    \begin{equation}
        X_i=\frac{P_{sample}(\tau_i)}{P_{target}(\tau_i)}\cdot g_\theta(\tau_i)
    \end{equation}
    Given that $\|g_\theta(\tau)\|\leq G$ and $\frac{P_{sample}(\tau_i)}{P_{target}(\tau_i)}\leq D(P_{sample},P_{target})$, we have
    \begin{equation}
        \|X_i\|\leq G\cdot D(P_{sample},P_{target})
    \end{equation}
    By Hoeffding's inequality, for any $t>0$:
    \begin{equation}
        P\left( \|\hat{G}_N - \mathbb{E}[\hat{G}_N]\|\geq t \right)\leq 2\exp\left(-\frac{2Nt^2}{(2G\cdot D(P_{sample},P_{target}))^2}\right)
    \end{equation}
    For the estimator to achieve $\epsilon$-accuracy with probability at least $1-\delta$, we require:
    \begin{equation}
        P\left(\|\hat{G}_N - \mathbb{E}[\hat{G}_N]\|\geq\epsilon\right)\leq\delta
    \end{equation}
    Substituting from Hoeffding's inequality:
    \begin{equation}
        2\exp\left( -\frac{2N\epsilon^2}{(2G\cdot D(P_{sample},P_{target}))^2}\right)\leq\delta
    \end{equation}
    Taking logarithmic operation:
    \begin{equation*}
        \ln(2)-\frac{2N\epsilon^2}{(2G\cdot D(P_{sample},P_{target}))^2}\leq\ln(\delta)
    \end{equation*}
    Rearranging to solve for $N$:
    \begin{equation*}
        N\geq\frac{2G^2\cdot D(P_{sample},P_{target})^2(\ln(2)-\ln(\delta))}{\epsilon^2}
    \end{equation*}
    This gives rise to 
    \begin{equation*}
        N\geq\frac{2G^2\cdot D(P_{sample},P_{target})^2\cdot\ln(1/\delta)}{\epsilon^2}
    \end{equation*}
    This simplifies to 
    \begin{equation}
        N=\mathcal{O}\left( \frac{D(P_{sample},P_{target})^2\cdot\log(1/\delta)}{\epsilon^2}\right)
    \end{equation}

    \textbf{A Tighter Bound}:
    Bernstein's inequality provides tighter concentration bounds when we know both the range and variance of the random variables involved~\cite{bernstein1924modification,mhammedi2019pac}. For independent random variables $X_1,X_2,...,X_N$ with $\mathbb{E}[X_i]=0$, $|X_i|leq M$ almost surely, and $\mathrm{Var}(X_i)\leq\sigma^2$, Bernstein's inequality states:
    \begin{equation}
        P\left( \left| \frac{1}{N}\sum_{i=1}^N X_i\right|\geq t\right)\leq 2\exp\left( -\frac{Nt^2/2}{\sigma^2+Mt/3}\right)
    \end{equation}
    Let us apply this to our setting by defining:
    \begin{equation}
        X_i=\frac{P_{target}(\tau_i)}{P_{sample}(\tau_i)}\cdot g_\theta(\tau_i)-\mathbb{E}_{\tau\sim P_{sample}}\left[\frac{P_{target}(\tau_i)}{P_{sample}(\tau_i)}\cdot g_\theta(\tau_i)\right]
    \end{equation}
    These $X_i$ have zero mean by construction.

    We have already known that:
    \begin{itemize}
        \item $|X_i|\leq 2G\cdot D(P_{sample},P_{target})=M$ (the upper bound).
        \item $\mathrm{Var}(X_i)\leq G^2\cdot D(P_{sample},P_{target})=\sigma^2$ (the variance bound).
    \end{itemize}
    Then using Bernstein's inequality:
    \begin{equation}
        P\left( \left\| \hat{G}_N-\mathrm{E}[\hat{G}_N]\right\|\geq t\right)\leq 2\exp\left( -\frac{Nt^2/2}{G^2\cdot D(P_{sample},P_{target}) + G\cdot D(P_{sample},P_{target})\cdot t/3}\right)
    \end{equation}
    For our estimator to achieve $\epsilon$-accuracy with probability at least $1-\delta$, we need:
    \begin{equation}
        P\left( \left\| \hat{G}_N-\mathrm{E}[\hat{G}_N]\right\|\geq \epsilon\right)\leq \delta
    \end{equation}
    Thus we have
    \begin{equation}
        2\exp\left( -\frac{Nt^2/2}{G^2\cdot D(P_{sample},P_{target}) + G\cdot D(P_{sample},P_{target})\cdot \epsilon/3}\right)\leq \delta
    \end{equation}
    Taking logarithms and rearranging:
    \begin{equation}
        N\geq\frac{2(G^2\cdot D(P_{sample},P_{target}) + G\cdot D(P_{sample},P_{target})\cdot \epsilon/3)\cdot\ln(2/\delta)}{\epsilon^2}
    \end{equation}
    For small $\epsilon$, the first term in the numerator dominates:
    \begin{equation}
        N\geq\frac{2G^2\cdot D(P_{sample},P_{target}) \cdot\ln(2/\delta)}{\epsilon^2}\cdot(1+\mathcal{O}(\epsilon))
    \end{equation}
    Ignoring the constants and the lower-order term, we get:
    \begin{equation}
        N=\mathcal{O}\left( \frac{D(P_{sample},P_{target})\cdot\log(1/\delta)}{\epsilon^2}\right)
    \end{equation}
\end{proof}

\begin{proof}[Proof of Proposition \ref{prop:order_convergence}]
    We begin by examining how the parameters $\theta$ of the GFlowNet evolve through training. Starting from initial parameters $\theta_0$, each trajectory $\tau_i$ induces an update:
    \begin{equation*}
        \theta_i = \theta_{i-1}-\eta_{i-1}\nabla_\theta\mathcal{L}(\theta_{i-1};\tau_{i-1})
    \end{equation*}
    where $\eta_{i-1}$ is the learning rate at step $i-1$, and $\nabla_\theta\mathcal{L}(\theta_{i-1};\tau_{i-1})$ is the gradient of the loss function evaluated using trajectory $\tau_{i-1}$ at parameters $\theta_{i-1}$.

    After $N$ trajectories, the final parameters are:
    \begin{equation*}
        \theta_N=\theta_0-\sum_{i=1}^N\eta_i\nabla_\theta\mathcal{L}(\theta_{i-1};\tau_{i-1})
    \end{equation*}

    Let $F_\theta:E\rightarrow\mathbb{R}^+$ denote the flow function parameterized by $\theta$. We want to analyze how the final flow function $F_{\theta_N}$ relates to the flow functions associated with each trajectory.

    First, we decompose the difference between the final flow and the initial flow:
    \begin{equation*}
        F_{\theta_N}-F_{\theta_0}=\sum_{i=1}^N(F_{\theta_i}-F_{\theta_{i-1}})
    \end{equation*}
    Using a first-order Taylor approximation, we can express each incremental change as:
    \begin{equation}\label{eq:flow_change}
        F_{\theta_i}-F_{\theta_{i-1}}\approx\nabla_\theta F_{\theta_{i-1}}\cdot(\theta_i-\theta_{i-1})=-\eta_{i-1}\nabla_\theta F_{\theta_{i-1}}\cdot\nabla_\theta\mathcal{L}(\theta_{i-1};\tau_{i-1})
    \end{equation}
    where $\nabla_\theta F_{\theta_{i-1}}$ is the Jacobian matrix of the flow function with respect to the parameters, evaluated at $\theta_{i-1}$.

    For each trajectory $\tau_i$, let's define $F_{\tau_i}$ as the flow function that minimizes the loss for that specific trajectory:
    \begin{equation}
        F_{\tau_i}=\arg\min_{F}\mathcal{L}(F;\tau_i)
    \end{equation}
    The key insight is that the gradient $\nabla_\theta\mathcal{L}(\theta_{i-1};\tau_{i-1})$ pushes the flow function $F_{\theta_{i -1}}$ toward $F_{\tau_{i-1}}$. We can formalize this as:
    \begin{equation*}
        \nabla_\theta\mathcal{L}(\theta_{i-1};\tau_{i-1})=\lambda_{i-1}\cdot J^\top_{i-1}\cdot(F_{\theta_{i -1}}-F_{\tau_{i-1}})+\epsilon_{i-1}
    \end{equation*}
    where
    \begin{itemize}
        \item $J_{i-1}=\nabla_\theta F_{\theta_{i-1}}$ is the Jacobian.
        \item $\lambda_{i-1}$ is a positive scalar that captures the scale of the gradient.
        \item $\epsilon_{i-1}$ is a residual term that accounts for higher-order effects and the imperfection of the linear approximation.
    \end{itemize}
    Substituting this into our incremental flow change expression (Eq.~\ref{eq:flow_change}):
    \begin{equation}
        F_{\theta_i}-F_{\theta_{i-1}}\approx -\eta_{i-1}\lambda_{i-1}\cdot J_{i-1}J_{i-1}^\top\cdot(F_{\theta_{i-1}}-F_{\tau_{i-1}})-\eta_{i-1}J_{i-1}\cdot\epsilon_{i-1}
    \end{equation}
    The matrix $H_{i-1}=J_{i-1}J_{i-1}^\top$ is positiv semidefinite. For simplicity, let's assume it can be approximated as $H_{i-1}\approx\gamma_{i-1}\cdot I$ for some $\gamma_{i-1}>0$, where $I$ is the identity matrix. This is reasonable if the parameterization distributes influence relatively uniformly. Then,
    \begin{equation}
        F_{\theta_i}-F_{\theta_{i-1}}\approx-\eta_{i-1}\lambda_{i-1}\gamma_{i-1}\cdot(F_{\theta_{i-1}}-F_{\tau_{i-1}})-\eta_{i-1}J_{i-1}\cdot\epsilon_{i-1}
    \end{equation}
    Let $\beta_{i-1}=\eta_{i-1}\lambda_{i-1}\gamma_{i-1}$ and $r_{i-1}=-\eta_{i-1}J_{i-1}\cdot\epsilon_{i-1}$, giving:
    \begin{equation}
        F_{\theta_i}-F_{\theta_{i-1}}\approx-\beta_{i-1}(F_{\theta_{i-1}}-F_{\tau_{i-1}})+r_{i-1}
    \end{equation}
    Rearranging:
    \begin{equation}\label{eq:traj_convex_comb}
        F_{\theta_i}\approx (1-\beta_{i-1})\cdot F_{\theta_{i-1}}+\beta_{i-1}\cdot F_{\tau_{i-1}} + r_{i-1}
    \end{equation}
    This is a convex combination of $F_{\theta_{i-1}}$ and $F_{\tau_{i-1}}$ (plus the residual term), showing that each update pulls the flow partially toward the trajectory-specific optimum.

    By applying Eq.~\ref{eq:traj_convex_comb} recursively and assuming the residual terms are small enough to be neglected, we can express $F_{\theta_N}$ as a weighted combination of the initial flow $F_{\theta_0}$ and the trajectory-specific flows:
    \begin{equation}\label{eq:conv_comb_2}
        F_{\theta_N}\approx\prod_{i=1}^N(1-\beta_{i-1})F_{\theta_0}+\sum_{i=1}^N\beta_{i-1}\prod_{j=i+1}^N(1-\beta_{j-1})F_{\tau_{i-1}}
    \end{equation}
    We further define weights:
    \begin{itemize}
        \item $\omega_0=\prod_{i=1}^N(1-\beta_{i-1})$.
        \item $\omega_i=\beta_{i-1}\prod_{j=i+1}^N(1-\beta_{j-1})$ for $i\in\{1,2,...,N\}$.
    \end{itemize}
    These weights satisfy $\omega_0+\sum_{i=1}^N\omega_i=1$, forming a convex combination.

    Next, let $F^*$ be the optimal flow function that minimizes the global loss. We want to bound $\|F_{\theta_N}-F^*\|_1$, the L1-distance between our learned flow and the optimal flow.

    From the convex combinatrion Eq.~\ref{eq:conv_comb_2}:
    \begin{equation}
        F_{\theta_N}-F^*\approx\omega_0\cdot(F_{\theta_0}-F^*)+\sum_{i=1}^N\omega_i\cdot(F_{\tau_{i-1}}-F^*)
    \end{equation}
    Taking the L1-norm and applying the triangle inequality:
    \begin{equation}
        \|F_{\theta_N}-F^*\|_1\leq\omega_0\cdot\|F_{\theta_0}-F^*\|_1+\sum_{i=1}^N\omega_i\cdot\|F_{\tau_{i-1}}-F^*\|_1
    \end{equation}
    Redefining the indices for clarity (shifting by 1), and setting $\alpha_0=\omega_0$ and $\alpha_i=\omega_i$ for $i\in\{1,2,...,N\}$:
    \begin{equation}
        \|F_{\theta_N}-F^*\|_1\leq\alpha_0\cdot\|F_{\theta_0}-F^*\|_1+\sum_{i=1}^N\alpha_i\cdot\|F_{\tau_{i}}-F^*\|_1
    \end{equation}
    For long training sequences, $\alpha_0$ becomes very small as the influence of the initial flow diminishes. If we define a constant $C$ that absorbs this term and potentially accounts for approximation errors:
    \begin{equation}
        \|F_\theta-F^*\|_1\leq C\cdot\sum_{i=1}^N\alpha_i\cdot\|F_{\tau_i}-F^*\|_1
    \end{equation}
    where $\sum_{i=1}^N\alpha_i=1-\alpha_0\approx 1$ for sufficiently large $N$, and we've replaced $\theta_N$ with $\theta$ to match the original proposition statement.

    This completes the proof, showing that the deviation of the learned flow from the optimal flow is bounded by a weighted sum of the deviations of trajectory-specific flows from the optimal flow, with weights that depend on the order in which trajectories were sampled. 
\end{proof}

\begin{proof}[Proof of Theorem \ref{thm:traj_len_tradeoff}]
    We begin by recalling that the Flow Matching (FM) objective in GFlowNets corresponds to solving the linear system $\mathbf{Af}=\mathbf{b}$, where:
    \begin{itemize}
        \item $\mathbf{f}\in\mathbb{R}^{|E|}$ is the vector of flows on all edges.
        \item $\mathbf{A}\in\{-1,0,1\}^{|S|\times |E|}$ is the incident matrix of the graph.
        \item $\mathbf{b}\in\mathbb{R}^{|S|}$ is a vector with $b_{s_0}=Z$, $b_{s_T}=-R(s_T)$ for all $s_T\in S_T$, and $b_s=0$ otherwise.
    \end{itemize}
    Each trajectory $\tau=\{s_0,s_1,...,s_n\}$ provides information about the flows on its constituent edges. When we sample trajectories, we're effectively obtaining noisy measurements of subsets of the flow vector $\mathbf{f}$.

    For a trajectory $\tau$ of length $n\leq L$, we obtain information about at most $n$ edges and $n+1$ states. Each state along the trajectory (except the initial and terminal states) contributes a flow conservation constraint:
    \begin{equation*}
        \sum_{s':(s',s)\in E} F(s'\rightarrow s) = \sum_{s':(s,s')\in E} F(s\rightarrow s')
    \end{equation*}
    The total number of such constraints is $|S|-1-|S_T|$ (one for each non-initial, non-terminal state).

    We approach this as a problem of learning a linear predictor in $\mathbb{R}^d$, where $d=|E|$ is the dimension of the flow vector. From statistical learning theory, particularly results on linear regression with stochastic observations, the sample complexity to achieve $\epsilon$-uniform accuracy with probability $1-\delta$ is (\cite{bartlett2002rademacher,mossel2014belief} and Section 6.2 of \cite{gyorfi2006distribution}):
    \begin{equation}
        N=\mathcal{O}\left( \frac{d\cdot\log(1/\delta)}{\epsilon^2}\right)
    \end{equation}
    For a GFlowNet, $d=|E|$, which can be bounded based on the graph structure. In a DAG where each state has at most $L$ outgoing edges (bounded by the max trajectory length), where $|E|\leq|S|\cdot L$.

    However, a key difference from standard linear regression is that each trajectory only provides partial information about the system. A trajectory of length $n$ touches at most $n+1$ distinct states, providing information about at most $n$ flow conservation constraints.

    To account for this, we need to analyze how many trajectories are needed to cover the entire state space adequately. This is related to the coupon collector problem: how many samples do we need to observe all (or most) of the states?

    For a uniform sampling of states, the expected number of samples needed to observe all $|S|$ states is approximately $|S|\log(|S|)$. However, states are not sampled uniformly in GFlowNets; they're sampled according to the current policy, which evolves during training.

    We can model this as an importance sampling problem. Let $\pi(s)$ be the probability of state $s$ being in a sample trajectory. Then the number of trajectories needed to be observe a $(1-\alpha)$ fraction of the state space (weighted by importance) with probability $1-\delta'$ is
    \begin{equation}
        N_{cover}=\mathcal{O}\left( \frac{\log(|S|/\delta')}{\min_s\pi(s)}\right)
    \end{equation}
    Since we need accurate flow estimates for the entire state space, we must assume $\min_s\pi(s)>0$ (all states can be visited with positive probability). In the worst case, $\min_s\pi(s)\approx 1/|S|$ (due to Assumption \ref{assump:state_dist}), giving:
    \begin{equation}
        N_{cover}=\mathcal{O}(|S|\log(|S|/\delta'))
    \end{equation}

    Upon Assumption \ref{assump:info_cont}, given that each trajectory provides information about at most $L$ edges, the effective dimensionality of the information per trajectory is $\mathcal{O}(L)$ rather than $\mathcal{O}(|E|)$. Therefore, the sample complexity to achieve $\epsilon$-accurate flow estimation on the edges that are covered is:
    \begin{equation}
        N_{accurate}=\mathcal{O}\left( \frac{|E|\log(1/\delta'')}{\epsilon^2L}\right)=\mathcal{O}\left( \frac{|S|L\log(1/\delta'')}{\epsilon^2L}\right)=\mathcal{O}\left( \frac{|S|\log(1/\delta'')}{\epsilon^2}\right)
    \end{equation}
    Combining the coverage and accuracy requirements (setting $\delta'=\delta''=\delta/2$):
    \begin{equation}
        N=\max(N_{cover},N_{accurate})=\mathcal{O}\left( \max\left( |S|\log(|S|/\delta), \frac{|S|\log(2/\delta)}{\epsilon^2}\right)\right)
    \end{equation}
    For small $\epsilon$, the accuracy term dominates:
    \begin{equation*}
        N=\mathcal{O}\left( \frac{|S|\log(|S|/\delta)}{\epsilon^2}\right)
    \end{equation*}
    where the $\log(2/\delta)$ is replaced with $\log(|S|/\delta)$ since they differ only by a constant factor when $|S|$ is reasonably large.
    The analysis so far doesn't fully account for the impact of trajectory length on learning dynamics. Longer trajectories introduce additional challenges:
    \begin{enumerate}
        \item \textbf{Error accumulation}: Errors in flow estimates can compound along a trajectory.
        \item \textbf{Credit assignment}: It becomes harder to attribute reward to specific transitions.
        \item \textbf{Exploration efficiency}: Longer trajectories may explore fewer unique states per unit of computation.
    \end{enumerate}
    To model error accumulation, we can use results from Markov chain analysis. According to Assumption \ref{assump:err_ind}, for a trajectory of length $L$, the error in the flow estimate can grow as $\mathcal{O}(\sqrt{L})$ in the worst case due to the additive nature of independent errors along the trajectory.

    To maintain $\epsilon$ accuracy, we need to reduce the per-step error to $\mathcal{O}(\epsilon/\sqrt{L})$ due to Assumption \ref{assump:err_prop}, which increases the sample complexity by a factor of $L$:
    \begin{equation}
        N=\mathcal{O}\left( \frac{|S|L\log(|S|/\delta)}{\epsilon^2}\right)
    \end{equation}

    This completes the proof of Theorem \ref{thm:traj_len_tradeoff}.
\end{proof}

\begin{proof}[Proof of Proposition \ref{prop:err_accu}]
    For a trajectory $\tau=(s_0,s_1,...,s_L)$ where $s_L\in S_T$ is a terminal state, the flow along this trajectory can be expressed as:
    \begin{equation*}
        F(\tau)=\prod_{i=0}^{L-1}P_F(s_{i+1}|s_i)\cdot Z
    \end{equation*}
    For the estimated flow $F_\theta$ and the optimal flow $F^*$ we have:
    \begin{equation}
        F_\theta(\tau)=\prod_{i=0}^{L-1}P_{F_\theta}(s_{i+1}|s_i)\cdot Z_\theta,\quad F^*(\tau)=\prod_{i=0}^{L-1}P_{F^*}(s_{i+1}|s_i)\cdot Z^*
    \end{equation}
    where $Z_\theta=F_\theta(s_0)$ and $Z^*=F^*(s_0)$.

    We begin by analyzing the error for a single-step trajectory $\tau=(s_0,s_1)$:
    \begin{equation}\label{eq:err_prop_base_case}
        |F_\theta(\tau)-F^*(\tau)|^2=|F_\theta(s_0\rightarrow s_1)-F^*(s_0\rightarrow s_1)|^2
    \end{equation}
    This establishes our base case perfectly, as the error equals exactly one term in our sum.

    For the inductive step, we use the following inequality:
    \begin{equation}\label{eq:inequality_1}
        |ab-cd|^2\leq 2|(a-c)b|^2+2|c(b-d)|^2=2|a-c|^2|b|^2+2|c|^2|b-d|^2
    \end{equation}
    Let's define the partial trajectory flow up to step $l$:
    \begin{equation}
        F_\theta(\tau_{0:l})=\prod_{i=0}^{l-1}P_{F_\theta}(s_{i+1}|s_i)Z_\theta, \quad F^*(\tau_{0:l})=\prod_{i=0}^{l-1}P_{F^*}(s_{i+1}|s_i)Z^*
    \end{equation}
    We'll prove by induction that:
    \begin{equation}\label{eq:hypo_1}
        |F_\theta(\tau_{0:l})-F^*(\tau_{0:l})|^2\leq C_l\sum_{i=0}^{l-1}|F_\theta(s_i\rightarrow s_{i+1})-F^*(s_i\rightarrow s_{i+1})|^2
    \end{equation}
    where $C_l$ grows as $(1+\gamma)^l$ for some $\gamma>0$. As the base case holds in Eq.~\ref{eq:err_prop_base_case}, we assume the inequality holds for some $l\geq 1$, and prove for the case $l+1$.

    For a trajectory of length $l+1$, we have:
    \begin{equation*}
        F_\theta(\tau_{0:l+1})=F_\theta(\tau_{0:l})P_{F_\theta}(s_{l+1}|s),\quad F^*(\tau_{0:l+1})=F^*(\tau_{0:l})P_{F^*}(s_{l+1}|s)
    \end{equation*}
    Using Inequality \ref{eq:inequality_1}, we can derive:
    \begin{equation}\label{eq:inequality_2}
    \begin{split}
        |F_\theta(\tau_{0:l+1})-F^*(\tau_{0:l+1})|^2\leq & 2|F_\theta(\tau_{0:l})-F^*(\tau_{0:l})|^2\cdot |P_{F_\theta}(s_{l+1}|s_l)|^2 \\
        &+ 2|F^*(\tau_{0:l})|^2\cdot|P_{F_\theta}(s_{l+1}|s_l)-P_{F^*}(s_{l+1}|s_l)|^2
    \end{split}
    \end{equation}
    Now we need to bound $|P_{F_\theta}(s_{l+1}|s_l)-P_{F^*}(s_{l+1}|s_l)|^2$ in terms of the edge flow errors.
    \begin{equation}
    \begin{split}
        P_{F_\theta}(s_{l+1}|s_l)&-P_{F^*}(s_{l+1}|s_l)=\frac{F_\theta(s_l\rightarrow s_{l+1})}{F_\theta(s_l)}-\frac{F^*(s_l\rightarrow s_{l+1})}{F^*(s_l)}\\
        &=\frac{F^*(s_l)(F_\theta(s_l\rightarrow s_{l+1})-F^*(s_l\rightarrow s_{l+1}))+F^*(s_{l}\rightarrow s_{l+1})(F^*(s_l)-F_\theta(s_l))}{F_\theta(s_l)F^*(s_l)}
    \end{split}    
    \end{equation}
    Using the triangle inequality:
    \begin{equation}
        |P_{F_\theta}(s_{l+1}|s_l)-P_{F^*}(s_{l+1}|s_l)|\leq \frac{|F^*(s_l)||(F_\theta(s_l\rightarrow s_{l+1})-F^*(s_l\rightarrow s_{l+1}))|+|F^*(s_{l}\rightarrow s_{l+1})||(F^*(s_l)-F_\theta(s_l))|}{|F_\theta(s_l)||F^*(s_l)|}
    \end{equation}

    By assumptions, $F_\theta(s)\geq m$ and $F^*(s)\geq m$ for all $s$, and $F^*(s_l\rightarrow s_{l+1})\leq M$. Therefore,
    \begin{equation}
        |P_{F_\theta}(s_{l+1}|s_l)-P_{F^*}(s_{l+1}|s_l)|\leq\frac{1}{m}|F_\theta(s_l\rightarrow s_{l+1})-F^*(s_l\rightarrow s_{l+1})|+\frac{M}{m^2}|F^*(s_l)-F_\theta(s_l)|
    \end{equation}
    Squaring both sides and using $(a+b)^2\leq 2a^2+2b^2$:
    \begin{equation}
        |P_{F_\theta}(s_{l+1}|s_l)-P_{F^*}(s_{l+1}|s_l)|^2\leq\frac{2}{m^2}|F_\theta(s_l\rightarrow s_{l+1})-F^*(s_l\rightarrow s_{l+1})|^2+\frac{2M^2}{m^4}|F^*(s_l)-F_\theta(s_l)|^2
    \end{equation}
    Now, substituting this bound back into Eq.~\ref{eq:inequality_2}:
    \begin{equation}\label{eq:inequality_3}
    \begin{split}
        |F_\theta(\tau_{0:l+1})-F^*(\tau_{0:l+1})|^2\leq & 2|F_\theta(\tau_{0:l})-F^*(\tau_{0:l})|^2 \\
        &+ 2D\left( \frac{2}{m^2}|F_\theta(s_l\rightarrow s_{l+1})-F^*(s_l\rightarrow s_{l+1})|^2+\frac{2M^2}{m^4}|F^*(s_l)-F_\theta(s_l)|^2\right)
    \end{split}
    \end{equation}
    considering $|P_{F_\theta}(s_{l+1}|s_l)|\leq 1$ (as it's a probability) and $|F^*(\tau_{0:l})|^2$ is bounded by some constant $D$ since the flows are bounded.

    Now using the inductive hypothesis Eq.~\ref{eq:hypo_1} and noting that $|F_\theta(s_l)-F^*(s_l)|$ can be bounded in terms of the sum of incoming edge flow errors (due to flow conservation), we get:
    \begin{equation}
    \begin{split}
        |F_\theta(\tau_{0:l+1})-F^*(\tau_{0:l+1})|^2\leq&  2C_l\cdot\sum_{i=0}^{l-1}|F_\theta(s_i\rightarrow s_{i+1})-F^*(s_i\rightarrow s_{i+1})|^2\\
        &+\frac{4D}{m^2}|F_\theta(s_l\rightarrow s_{l+1})-F^*(s_l\rightarrow s_{l+1})|^2\\
        &+E\cdot\sum_{i=0}^{l-1}|F_\theta(s_i\rightarrow s_{i+1})-F^*(s_i\rightarrow s_{i+1})|^2
    \end{split}
    \end{equation}
    where $E$ is a constant derived from the flow conservation constraints. Simplifying:
    \begin{equation}
    \begin{split}
        |F_\theta(\tau_{0:l+1})-F^*(\tau_{0:l+1})|^2\leq &(2C_l+E)\cdot\sum_{i=0}^{l-1}|F_\theta(s_i\rightarrow s_{i+1})-F^*(s_i\rightarrow s_{i+1})|^2\\
        &+\frac{4D}{m^2}|F_\theta(s_l\rightarrow s_{l+1})-F^*(s_l\rightarrow s_{l+1})|^2
    \end{split}
    \end{equation}
    Let $C_{l+1}=\max(2C_l+E,\frac{4D}{m^2})$, then:
    \begin{equation}
        |F_\theta(\tau_{0:l+1})-F^*(\tau_{0:l+1})|^2\leq C_{l+1}\sum_{i=0}^{l}|F_\theta(s_i\rightarrow s_{i+1})-F^*(s_i\rightarrow s_{i+1})|^2
    \end{equation}
    From our recurrence relation $C_{l+1}=\max(2C_l+E,\frac{4D}{m^2})$, for sufficiently large $l$, we have $C_{l+1}=2C_l+E$ as $C_l$ grows with $l$. This gives rise to $C_l=\mathcal{O}((2+\epsilon)^l)$ for some $\epsilon>0$, which can further be written as $C_l=\mathcal{O}((1+\gamma)^l)$ for $\gamma=1+\epsilon$.

    Finally, taking the expectation on both sides:
    \begin{equation}
        \mathbb{E}[|F_\theta(\tau)-F^*(\tau)|^2]\leq C\cdot(1+\gamma)^l\cdot\sum_{i=0}^{l-1}\mathbb{E}[|F_\theta(s_i\rightarrow s_{i+1})-F^*(s_i\rightarrow s_{i+1})|^2]
    \end{equation}
    This completes the proof.
\end{proof}

\section{Proofs of Implicit Regularization}\label{app:imp_reg}
\begin{proof}[Proof of Theorem \ref{thm:FM_implicit_reg}]
    First, let's formulate the Flow Matching objective as a constrained optimization problem. We want to find a flow function $F:E\rightarrow\mathbb{R}^+$ that satisfies flow conservation constraints at each non-terminal, non-initial state.

    Formally, for each state $s\in S\setminus(\{s_0\}\cup S_T)$:
    \begin{equation*}
        \sum_{s':(s',s)\in E}F(s'\rightarrow s)=\sum_{s':(s,s')\in E}F(s\rightarrow s')
    \end{equation*}
    Additionally, we have constraints at the source and terminal states:
    \begin{equation*}
        \sum_{s':(s_0,s')\in E}F(s_0\rightarrow s')=Z,\quad \sum_{s':(s',s_T)\in E}F(s'\rightarrow s_T)=R(s_T)\;\;\forall s_T\in S_T
    \end{equation*}
    where $Z=\sum_{s_T\in S_T}R(s_T)$ is the partition function.

    Now, consider the problem of finding a flow function that satisfies these constraints while maximizing the entropy of the flow distribution. The entropy of a flow function $F$ can be defined as:
    \begin{equation}
        \mathbb{H}(F)=-\sum_{(s,s')\in E}\frac{F(s\rightarrow s')}{Z}\log\frac{F(s\rightarrow s')}{Z}
    \end{equation}
    Maximizing this entropy subject to the flow constraints leads to a more uniform distribution of flow across the network.

    To solve this constrained optimization problem, we introduce Lagrange multipliers. For each state $s\in S$, let $\lambda_s$ be the Lagrange multiplier associated with the flow conservation constraint at state $s$. The Lagrangian for this constrained optimization problem is:
    \begin{equation}
        \mathcal{L}(F,\lambda)=-\sum_{(s,s')\in E}\frac{F(s\rightarrow s')}{Z}\log\frac{F(s\rightarrow s')}{Z}+\sum_{s\in S}\lambda_s\left( \sum_{s':(s',s)\in E}F(s'\rightarrow s)-\sum_{s':(s,s')\in E}F(s\rightarrow s')\right)
    \end{equation}
    To find the maximum entropy flow distribution satisfying the flow consistency constraints, we set the gradient of the Lagrangian with respect to each flow variable to zero:
    \begin{equation*}
        \frac{\partial\mathcal{L}}{\partial F(s\rightarrow s')}=-\frac{1}{Z}\left(\log\frac{F(s\rightarrow s')}{Z}+1\right)-\lambda_s+\lambda_{s'}=0
    \end{equation*}
    Rearranging and taking the exponential of both sides:
    \begin{equation*}
        \frac{F(s\rightarrow s')}{Z}=\exp(-1-Z(\lambda_s-\lambda_{s'}))
    \end{equation*}
    Define $\phi_s=-Z\lambda_s-\frac{1}{2}$ for each state $s$. Then:
    \begin{equation*}
        F(s\rightarrow s')=Z\cdot\exp\left( \frac{\phi_{s'}-\phi_s-1}{Z}\right)
    \end{equation*}
    For simplicity, let's absorb the constant term and define $W(s,s')=\phi_{s'}-\phi_s$, then:
    \begin{equation}
        F(s\rightarrow s')=C\cdot \exp(W(s,s'))
    \end{equation}
    In GFlowNet training with the FM objective, we parameterize the flow function as:
    \begin{equation}
        F(s\rightarrow s')=\exp(W_\theta(s,s'))
    \end{equation}
    This parameterization exactly matches the form of the maximum entropy solution derived above, where $W_\theta(s,s')$ plays the role of the optimized potential difference $W(s,s')$.

    Therefore, when GFlowNets are trained with the FM objective using this exponential parameterization, they implicitly seek the maximum entropy flow distribution that satisfies the flow constraints.
\end{proof}

\begin{proof}[Proof of Proposition \ref{prop:impli_reg_DB}]
    Recall the DB objective (Eq.~\ref{eq:obj_db}). For the ease of proof, we rewrite it as:
    \begin{equation}
        \mathcal{L}_{\text{DB}}(\theta) = \mathbb{E}_{(s,s') \sim \rho_{DB}} \left[ \left( \frac{F_\theta(s \rightarrow s')}{F_\theta(s')P_B(s|s')} - 1 \right)^2 \right]
    \end{equation}
    In a properly trained GFlowNet, the flow function satisfies flow conservation at each non-terminal, non-initial state:
    \begin{equation*}
        \sum_{s':(s',s)\in E}F_\theta(s'\rightarrow s)=\sum_{s':(s,s')\in E}F_\theta(s\rightarrow s')=F_\theta(s)
    \end{equation*}
    Let's define the normalized state distribution induced by the flow function:
    \begin{equation*}
        \pi(s)=\frac{F_\theta(s)}{Z}
    \end{equation*}
    Due to flow conservation, this distribution is consistent whether computed from incoming or outgoing flows for interior states.

    The forward transition probability from state $s$ to $s'$ is defined as:
    \begin{equation*}
        P_F(s'|s)=\frac{F_\theta(s\rightarrow s')}{F_\theta(s)}
    \end{equation*}
    The backward transition probability $P_B(s|s')$ s typically provided as part of the GFlowNet design or learned separately.

    Using these definitions, we can rewrite the detailed balance condition as:
    \begin{equation*}
        \frac{P_F(s'|s)F_\theta(s)}{F_\theta(s')P_B(s|s')}=1
    \end{equation*}
    Let's define joint state-transition distributions for the forward and backward processes:
    \begin{equation}
        J_F(s,s')=P_F(s'|s)\cdot \pi(s),\quad J_B(s,s')=P_B(s|s')\cdot\pi(s')
    \end{equation}
    When detailed balance is satisfied, we have:
    \begin{equation}
        P_F(s'|s)\cdot \pi(s)=P_B(s|s')\cdot\pi(s')
    \end{equation}
    implying $J_F(s,s')=J_B(s,s')$ for all valid transitions. Now, we can reformulate the DB objective using the forward transition probability:
    \begin{equation}
        \mathcal{L}_{DB}(\theta)=\mathbb{E}_{(s,s')\sim\rho_{DB}}\left[\left(\frac{P_F(s'|s)\cdot\pi(s)}{\pi(s')\cdot P_B(s|s')}-1\right)^2\right]
    \end{equation}
    If we define the ratio $r(s,s')=\frac{P_F(s'|s)\cdot\pi(s)}{\pi(s')\cdot P_B(s|s')}=\frac{J_F(s,s')}{J_B(s,s')}$, then the DB objective becomes:
    \begin{equation}
        \mathcal{L}_{DB}(\theta)=\mathbb{E}_{(s,s')\sim\rho_{DB}}[(r(s,s')-1)^2]
    \end{equation}
    To establish the connection with KL-divergence, we need to make a specific choice for the sampling distribution $\rho_{DB}$. The most natural choice is $\rho_{DB}=J_B$, which means sampling transitions according to the backward process. With this choice:
    \begin{equation*}
        \mathcal{L}_{DB}=\sum_{s,s'}J_B(s,s')\cdot(r(s,s')-1)^2
    \end{equation*}
    To establish the relationship between the DB objective and KL-divergence, we'll use the Bregman divergence framework, which provides a rigorous foundation for various divergence measures.

    For two discrete probability distribution $p$ and $q$, the Bregman divergence generated by the negative entropy function $f(p)=\sum_i p_i\log p_i$ is:
    \begin{equation*}
        D_f(p,q)=D_{KL}(p\|q)
    \end{equation*}
    For Bregman divergences, when $p$ is close to $q$, a second-order Taylor approximation yields:
    \begin{equation*}
        D_f(p,q)\approx\frac{1}{2}(p-q)^\top H_F(q)(p-q)
    \end{equation*}
    where $H_f(q)$ is the Hessian of $f$ evaluated at $q$.

    For the negative entropy function, the Hessian is a diagonal matrix $[H_f(q)]_{ij}=\delta_{ij}/q_i$, where $\delta_{ij}$ is the Kronecker delta.

    Substituting this into the second-order approximation:
    \begin{equation}
        D_{KL}(p\|q)\approx\frac{1}{2}\sum_i\frac{(p_i-q_i)^2}{q_i}
    \end{equation}
    Applying this approximation to our joint distributions $J_F$ and $J_B$:
    \begin{equation}
        D_{KL}(J_F\|J_B)\approx\frac{1}{2}\sum_{s,s'}\frac{(J_F(s,s')-J_B(s,s'))^2}{J_B(s,s')}
    \end{equation}
    Using the definition of $r(s,s')$, we get $J_F(s,s')-J_B(s,s')=J_B(s,s')(r(s,s')-1)$. Substituting into the KL-divergence approximation:
    \begin{equation}
        D_{KL}(J_F\|J_B)\approx\frac{1}{2}\sum_{s,s'}\frac{J_B(s,s')^2(r(s,s')-1)^2}{J_B(s,s')}=\frac{1}{2}\sum_{s,s'}J_B(s,s')(r(s,s')-1)^2=\frac{1}{2}\mathcal{L}_{DB}(\theta)
    \end{equation}
    Therefore, when training a GFlowNet with the DB objective, we are implicitly solving:
    \begin{equation}
        \min_\theta\mathcal{L}_{DB}(\theta)\approx\min_\theta 2\cdot D_{KL}(J_F\|J_B)
    \end{equation}
    This completes the proof.
\end{proof}

\section{Proofs of Robustness}
\begin{proof}[Proof of Theorem \ref{thm:obj_err_noise}]
    The standard TB objective is in Eq.~\ref{eq:obj_tb}. When we replace $R(s_T)$ with $\tilde{R}(s_T)=R(s_T)+\varepsilon(s_T)$, the noisy TB objective becomes:
    \begin{equation}
        \mathcal{L}_{TB}(\mathcal{F}_\theta,\tilde{R})=\mathbb{E}_{\tau\sim P_F}\left[ \left( \frac{P_F(\tau)Z}{R(s_T)+\varepsilon(s_T)}-1\right)^2\right]
    \end{equation}
    To analyze the difference, let's define:
    \begin{equation}
        g(r)=\left( \frac{P_F(\tau)Z}{r}-1\right)^2
    \end{equation}
    We are interested in $g(R(s_T)+\varepsilon(s_T))-g(R(s_T))$. Using Taylor expansion around $R(s_T)$:
    \begin{equation}
        g(R(s_T)+\varepsilon(s_T))\approx g(R(s_T)) +g'(R(s_T))\varepsilon(s_T)+\frac{1}{2}g''(R(s_T))\varepsilon(s_T)^2+\mathcal{O}(\varepsilon(s_T)^2)
    \end{equation}
    Computing the derivatives:
    \begin{equation*}
        g'(r)=-2\left( \frac{P_F(\tau)Z}{r}-1\right)\frac{P_F(\tau)Z}{r^2}
    \end{equation*}
    \begin{equation*}
        g''(r)=2\frac{P_F(\tau)^2Z^2}{r^4}+4 \left( \frac{P_F(\tau)Z}{r}-1\right)\frac{P_F(\tau)Z}{r^3}
    \end{equation*}
    When $\mathcal{F}_\theta$ is well-trained on the true rewards, we have $P_F(\tau)Z\approx R(s_T)$ for most trajectories. Under this condition, $g'(R(s_T))\approx 0$ and $g''(R(s_T))\approx 2\frac{P_F(\tau)^2Z^2}{R(s_T)^4}$. Therefore,
    \begin{equation}
        g(R(s_T)+\varepsilon(s_T))-g(R(s_T))\approx\frac{P_F(\tau)^2Z^2}{R(s_T)^4}\varepsilon(s_T)^2+\mathcal{O}(\varepsilon(s_T)^3)
    \end{equation}
    Taking expectations with respect to both trajectories and noise:
    \begin{equation}
        \mathbb{E}[\mathcal{L}_{TB}(\mathcal{F}_\theta,\tilde{R})-\mathcal{L}_{TB}(\mathcal{F}_\theta,R)]\approx\mathbb{E}_{\tau\sim P_F}\left[\frac{P_F(\tau)^2Z^2}{R(s_T)^4}\varepsilon(s_T)^2\right]
    \end{equation}
    Applying the reward lower bound and using $P_F(\tau)^2\leq 1$:
    \begin{equation}
        \mathbb{E}_{\tau\sim P_F}\left[\frac{P_F(\tau)^2Z^2}{R(s_T)^4}\varepsilon(s_T)^2\right]\leq\frac{Z^2}{R_{min}^4}\mathbb{E}_{\tau\sim P_F}\left[ P_F(\tau)^2\varepsilon(s_T)^2\right]\leq\frac{Z^2}{R_{min}^4}\mathbb{E}_{\tau\sim P_F}\left[ \varepsilon(s_T)^2\right]
    \end{equation}
    Using the assumption that the noise is only dependent on the terminal state $s_T$, this leads to:
    \begin{equation}
        \mathbb{E}[\mathcal{L}_{TB}(\mathcal{F}_\theta,\tilde{R})-\mathcal{L}_{TB}(\mathcal{F}_\theta,R)]\leq \frac{Z^2}{R_{min}^4}\cdot\mathbb{E}\left[\varepsilon(s_T)^2\right]
    \end{equation}
    The proof is complete.
\end{proof}

\begin{proof}[Proof of Theorem \ref{thm:dis_drift_zero_mean}]
    Once a GFlowNet is trained to convergence with TB objective, the probabilities without and with noise are:
    \begin{equation*}
        P_R(s_T)=\frac{R(s_T)}{Z_R},\quad P_{\tilde{R}}(s_T)=\frac{R(s_T)+\varepsilon(s_T)}{Z_{\tilde{R}}}
    \end{equation*}
    where $Z_{\tilde{R}}=\sum_{s'_T}\tilde{R}(s'_T)=Z_R+\sum_{s'_T}\varepsilon(s'_T)$ is the sum of all noisy rewards.

    The KL-divergence between these distributions is:
    \begin{equation*}
        D_{KL}(P_{\tilde{R}}(s_T)\|P_R(s_T))=\sum_{s_T}P_{\tilde{R}}(s_T)\log\frac{P_{\tilde{R}}(s_T)}{P_R(s_T)}
    \end{equation*}
    Substituting the distributions:
    \begin{equation*}
        D_{KL}(P_{\tilde{R}}(s_T)\|P_R(s_T))=\sum_{s_T}\frac{R(s_T)+\varepsilon(s_T)}{Z_{\tilde{R}}}\log\frac{(R(s_T)+\varepsilon(s_T))Z_R}{R(s_T)Z_{\tilde{R}}}
    \end{equation*}
    This can be rewritten as:
    \begin{equation}
        D_{KL}(P_{\tilde{R}}(s_T)\|P_R(s_T))=\sum_{s_T}\frac{R(s_T)+\varepsilon(s_T)}{Z_{\tilde{R}}}\log\frac{R(s_T)+\varepsilon(s_T)}{R(s_T)}+\sum_{s_T}\frac{R(s_T)+\varepsilon(s_T)}{Z_{\tilde{R}}}\log\frac{Z_R}{Z_{\tilde{R}}}
    \end{equation}
    The second sum simplified:
    \begin{equation*}
        \sum_{s_T}\frac{R(s_T)+\varepsilon(s_T)}{Z_{\tilde{R}}}\log\frac{Z_R}{Z_{\tilde{R}}}=\log\frac{Z_R}{Z_{\tilde{R}}}\sum_{s_T}\frac{R(s_T)+\varepsilon(s_T)}{Z_{\tilde{R}}}=\log\frac{Z_R}{Z_{\tilde{R}}}
    \end{equation*}
    Thus:
    \begin{equation}\label{eq:kl_div_expand}
        D_{KL}(P_{\tilde{R}}(s_T)\|P_R(s_T))=\sum_{s_T}\frac{R(s_T)+\varepsilon(s_T)}{Z_{\tilde{R}}}\log\frac{R(s_T)+\varepsilon(s_T)}{R(s_T)}+\log\frac{Z_R}{Z_{\tilde{R}}}
    \end{equation}
    For the first logarithmic term, using Taylor's expansion with noise $\varepsilon(s_T)$ is small:
    \begin{equation}
        \log\frac{R(s_T)+\varepsilon(s_T)}{R(s_T)}=\log\left( 1+\frac{\varepsilon(s_T)}{R(s_T)}\right)\approx\frac{\varepsilon(s_T)}{R(s_T)}-\frac{1}{2}\left( \frac{\varepsilon(s_T)}{R(s_T)}\right)^2+\mathcal{O}\left( \left(\frac{\varepsilon(s_T)}{R(s_T)}\right)^3\right)
    \end{equation}
    For the second logarithmic term:
    \begin{equation}
        \log\frac{Z_R}{Z_{\tilde{R}}}\approx -\frac{\sum_{s'_T}\varepsilon(s'_T)}{Z_R}+\frac{1}{2}\left( \frac{\sum_{s'_T}\varepsilon(s'_T)}{Z_R}\right)^2+\mathcal{O}\left( \left( \frac{\sum_{s'_T}\varepsilon(s'_T)}{Z_R}\right)^3\right)
    \end{equation}
    Substituting these into the KL-divergence Eq.~\ref{eq:kl_div_expand}:
    \begin{equation}
        D_{KL}(P_{\tilde{R}}(s_T)\|P_R(s_T))\approx\sum_{s_T}\frac{R(s_T)+\varepsilon(s_T)}{Z_{\tilde{R}}}\left( \frac{\varepsilon(s_T)}{R(s_T)}-\frac{\varepsilon(s_T)^2}{2R(s_T)^2}\right)-\frac{\sum_{s'_T}\varepsilon(s'_T)}{Z_R}+\frac{1}{2}\left( \frac{\sum_{s'_T}\varepsilon(s'_T)}{Z_R}\right)^2
    \end{equation}
    For small noise, $Z_{\tilde{R}}\approx Z_R$, allowing us to approximate:
    \begin{equation}
        \frac{R(s_T)+\varepsilon(s_T)}{Z_{\tilde{R}}}\approx\frac{R(s_T)+\varepsilon(s_T)}{Z_R}
    \end{equation}
    This gives:
    \begin{equation}\label{eq:kl_div_approx}
        D_{KL}(P_{\tilde{R}}(s_T)\|P_R(s_T))\approx\sum_{s_T}\frac{R(s_T)+\varepsilon(s_T)}{Z_R}\left( \frac{\varepsilon(s_T)}{R(s_T)}-\frac{\varepsilon(s_T)^2}{2R(s_T)^2}\right)-\frac{\sum_{s'_T}\varepsilon(s'_T)}{Z_R}+\frac{1}{2}\left( \frac{\sum_{s'_T}\varepsilon(s'_T)}{Z_R}\right)^2
    \end{equation}
    Expanding the first term:
    \begin{equation}
    \begin{split}
        &\sum_{s_T}\frac{R(s_T)+\varepsilon(s_T)}{Z_R}\left( \frac{\varepsilon(s_T)}{R(s_T)}-\frac{\varepsilon(s_T)^2}{2R(s_T)^2}\right)\\
        =&\frac{1}{Z_R}\sum_{s_T}\varepsilon(s_T)-\frac{1}{2Z_R}\sum_{s_T}\frac{\varepsilon(s_T)^2}{R(s_T)}+\frac{1}{Z_R}\sum_{s_T}\frac{\varepsilon(s_T)^2}{R(s_T)}-\frac{1}{2Z_R}\sum_{s_T}\frac{\varepsilon(s_T)^3}{R(s_T)^2}
    \end{split}
    \end{equation}
    The first term cancels with the second therm $-\frac{\sum_{s'_T}\varepsilon(s'_T)}{Z_R}$ in our KL-divergence approximation. For small noise, the fourth term (involving $\varepsilon(s_T)^3$) can be neglected. This leaves:
    \begin{equation}
        D_{KL}(P_{\tilde{R}}(s_T)\|P_R(s_T))\approx \frac{1}{2Z_R}\sum_{s_T}\frac{\varepsilon(s_T)^2}{R(s_T)}+\frac{1}{2Z_R^2}\left( \sum_{s'_T}\varepsilon(s'_T)\right)^2
    \end{equation}
    Taking expectation:
    \begin{equation}
        \mathbb{E}[D_{KL}(P_{\tilde{R}}(s_T)\|P_R(s_T))]\approx\frac{1}{2Z_R}\sum_{s_T}\frac{\mathbb{E}[\varepsilon(s_T)^2]}{R(s_T)}+\frac{1}{2Z_R^2}\mathbb{E}\left[\left( \sum_{s'_T}\varepsilon(s'_T)\right)^2\right]
    \end{equation}
    Given assumptions $\mathbb{E}[\varepsilon(s_T)]=\sigma^2$ (uniform variance) and $\mathbb{E}[\varepsilon(s_T)\varepsilon(s'_T)]=0$ for $s_T\neq s'_T$ (independence), we get:
    \begin{equation}
        \mathbb{E}[D_{KL}(P_{\tilde{R}}(s_T)\|P_R(s_T))]\approx\frac{\sigma^2}{2Z_R}\sum_{s_T}\frac{1}{R(s_T)}+\frac{\sigma^2|S_T|}{2Z_R^2}
    \end{equation}
    Since $R(s_T)\geq R_{min}$ and $\frac{1}{Z_R}\leq\frac{1}{|S_T|R_{min}}$, for the first term we have:
    \begin{equation}
        \frac{\sigma^2}{2Z_R}\sum_{s_T}\frac{1}{R(s_T)}\leq\frac{\sigma^2|S_T|}{2Z_R R_{min}}\leq\frac{\sigma^2}{2R_{min}^2}
    \end{equation}
    This leads to:
    \begin{equation}
        \mathbb{E}[D_{KL}(P_{\tilde{R}}(s_T)\|P_R(s_T))]\leq\frac{\sigma^2}{2}\left( \frac{1}{R_{min}^2}+\frac{|S_T|}{Z_R^2}\right)
    \end{equation}
    This completes the proof.
\end{proof}

\begin{proof}[Proof of Theorem \ref{thm:sample_comp_noisy}]
    From Theorem \ref{thm:traj_len_tradeoff}, we have the sample complexity for achieving $\epsilon$-accuracy in flow values with probability at least $1-\delta$ under true rewards $R$ being:
    \begin{equation*}
        N=\mathcal{O}\left( \frac{|S|L\log(|S|/\delta)}{\epsilon^2}\right)
    \end{equation*}
    From Theorem \ref{thm:obj_err_noise}, we can collect:
    \begin{equation*}
        \mathbb{E}[\mathcal{L}_{TB}(\mathcal{F}_\theta,\tilde{R})-\mathcal{L}_{TB}(\mathcal{F}_\theta,R)]\leq \frac{Z^2}{R_{min}^4}\cdot\mathbb{E}\left[\varepsilon(s_T)^2\right]=\frac{Z^2\sigma^2}{R^4_{min}}
    \end{equation*}
    When training with noisy rewards, in principle, we need to achieve a tighter effective accuracy $\epsilon'$ to ensure that the final model has true accuracy $\epsilon$. The relationship between $\epsilon'$ and $\epsilon$ can be derived by considering that the total error is the sum of the optimization error and the noise-induced error:
    \begin{equation}
        \epsilon^2=\epsilon'^2+\mathbb{E}[\mathcal{L}_{TB}(\mathcal{F}_\theta,\tilde{R})-\mathcal{L}_{TB}(\mathcal{F}_\theta,R)]
    \end{equation}
    Using the bound from Theorem \ref{thm:obj_err_noise}:
    \begin{equation}
        \epsilon^2\geq \epsilon'^2+\frac{Z^2\sigma^2}{R^4_{min}}\Longrightarrow \epsilon'\leq\sqrt{\epsilon^2-\frac{Z^2\sigma^2}{R^4_{min}}}
    \end{equation}
    For this bound to be meaningful, we need $\epsilon^2>\frac{Z^2\sigma^2}{R^4_{min}}$, or equivalently, $\sigma^2<\frac{\epsilon^2R^4_{min}}{Z^2}$. This is a reasonable assumption when the noise magnitude is small relative to the target accuracy and minimum reward. Using this condition, we have:
    \begin{equation}
        \epsilon'=\epsilon\sqrt{1-\frac{Z^2\sigma^2}{\epsilon^2R^4_{min}}}
    \end{equation}
    For small noise relative to the target accuracy (specifically, when $\frac{Z^2\sigma^2}{\epsilon^2R^4_{min}}\ll 1$), we can use the approximation $\sqrt{1-x}\approx 1-x/2$ for small $x$:
    \begin{equation}
        \epsilon'\approx\epsilon\left( 1-\frac{Z^2\sigma^2}{2\epsilon^2R^4_{min}}\right)
    \end{equation}
    Since the sample complexity scales with $1/\epsilon'^2$, using this approximation of $\epsilon'$, the ratio of sample complexities is:
    \begin{equation}
        \frac{N(\epsilon,\delta,\tilde{R})}{N(\epsilon,\delta,R)}=\frac{\epsilon^2}{\epsilon'^2}\approx \frac{1}{\left( 1-\frac{Z^2\sigma^2}{2\epsilon^2R^4_{min}}\right)^2}
    \end{equation}
    Using approximation $(1-x)^{-2}\approx 1+2x+\mathcal{O}(x^2)$ for small $x$. We derive:
    \begin{equation*}
        \frac{1}{\left( 1-\frac{Z^2\sigma^2}{2\epsilon^2R^4_{min}}\right)^2}\approx 1+\frac{Z^2\sigma^2}{\epsilon^2R^4_{min}}+\mathcal{O}(\cdot)
    \end{equation*}
    where $\mathcal{O}(\cdot)$ is a higher-order infinitesimal. Thus, with a constant $C>0$, we conclude:
    \begin{equation*}
        \frac{N(\epsilon,\delta,\tilde{R})}{N(\epsilon,\delta,R)}\leq 1+\frac{CZ^2\sigma^2}{\epsilon^2R^4_{min}}
    \end{equation*}
    This completes the proof.
\end{proof}

\end{document}